\documentclass{article}

    \usepackage[final]{neurips_2020}

\usepackage[utf8]{inputenc} 
\usepackage[T1]{fontenc}    
\usepackage{hyperref}       
\usepackage{url}            
\usepackage{booktabs}       
\usepackage{amsfonts}       
\usepackage{nicefrac}       
\usepackage{microtype}      

\usepackage{amstext} 
\usepackage{array}   
\newcolumntype{L}{>{$}l<{$}} 
\newcolumntype{C}{>{$}c<{$}} 
\usepackage{siunitx} 
\usepackage{multirow,subfig}

\usepackage{wrapfig,lipsum,booktabs}

\newcommand*{\affmark}[1][*]{\textsuperscript{#1}}
\newcommand*\samethanks[1][\value{footnote}]{\footnotemark[#1]}

\title{Implicit Graph Neural Networks}

\author{
  Fangda Gu\affmark[1]\thanks{Equal contributions. Work done during Heng's visit to University of California at Berkeley.} \\
  \texttt{gfd18@berkeley.edu}
  \And
  Heng Chang\affmark[2]\samethanks \\
  \texttt{changh17@mails.tsinghua.edu.cn} \\
  \AND
  Wenwu Zhu\affmark[3] \\
  \texttt{wwzhu@tsinghua.edu.cn} \\
  \And
  Somayeh Sojoudi\affmark[1,2] \\
  \texttt{sojoudi@berkeley.edu} \\
  \And
  Laurent {El Ghaoui}\affmark[1,2] \\
  \texttt{elghaoui@berkeley.edu} \\
}

\usepackage{lipsum}
\usepackage{algorithmic}
\usepackage{amsfonts}
\usepackage{amsmath}
\usepackage{amssymb}
\usepackage{amsthm}
\usepackage{amsopn}
\usepackage{tikz}
\usetikzlibrary{arrows,calc,positioning,shadows,shapes,matrix,fit}
\usepackage{verbatim}
\usepackage{todonotes}

\newcommand{\BEAS}{\begin{eqnarray*}}
\newcommand{\EEAS}{\end{eqnarray*}}
\newcommand{\BEA}{\begin{eqnarray}}
\newcommand{\EEA}{\end{eqnarray}}
\newcommand{\BEQ}{\begin{equation}}
\newcommand{\EEQ}{\end{equation}}
\newcommand{\BIT}{\begin{itemize}}
\newcommand{\EIT}{\end{itemize}}
\newcommand{\BNUM}{\begin{enumerate}}
\newcommand{\ENUM}{\end{enumerate}}

\newcommand{\BA}{\begin{array}}
\newcommand{\EA}{\end{array}}

\newcommand{\eg}{{\it e.g.}}
\newcommand{\ie}{{\it i.e.}}
\newcommand{\etc}{{\it etc.}}

\newcommand{\reals}{{\mathbb R}}

\newcommand{\diag}{\mathop{\bf diag}}

\newcommand{\vecc}{\mathop{\bf vec}}

\newcommand{\argmin}{\mathop{\rm argmin}}

\newtheorem{definition}{Definition}[section]
\newtheorem{theorem}{Theorem}[section]
\newtheorem{lemma}{Lemma}[section]
\newtheorem{remark}{Remark}[section]

\def\lpf{\lambda_{\rm pf}}

\usepackage{enumitem}

\begin{document}

\maketitle

\vspace{-2.5em}
 \begin{center}
  \affmark[1]Department of Electrical Engineering and Computer Sciences, University of California at Berkeley \\ \affmark[2]Tsinghua-Berkeley Shenzhen Institute, Tsinghua University\\ \affmark[3]Department of Computer Science and Technology, Tsinghua University
 \end{center}
 \vspace{1em}

\begin{abstract}
  Graph Neural Networks (GNNs) are widely used deep learning models that learn meaningful representations from graph-structured data. Due to the finite nature of the underlying recurrent structure, current GNN methods may struggle to capture long-range dependencies in underlying graphs. To overcome this difficulty, we propose a graph learning framework, called Implicit Graph Neural Networks (IGNN\footnote{Code available at \url{https://github.com/SwiftieH/IGNN}.}), where predictions are based on the solution of a fixed-point equilibrium equation involving implicitly defined ``state'' vectors. We use the Perron-Frobenius theory to derive sufficient conditions that ensure well-posedness of the framework. Leveraging implicit differentiation, we derive a tractable projected gradient descent method to train the framework. Experiments on a comprehensive range of tasks show that IGNNs consistently capture long-range dependencies and outperform the state-of-the-art GNN models. 
\end{abstract}
\section{Introduction}

Graph neural networks (GNNs) \citep{zhou2018graph,zhang2020deep} have been widely used on graph-structured data to obtain a meaningful representation of nodes in the graph. By iteratively aggregating information from neighboring nodes, GNN models encode graph-relational information into the representation, which then benefits a wide range of tasks, including biochemical structure discovery \citep{gilmer2017neural,wan2019neodti}, computer vision~\citep{kampffmeyer2018rethinking}, and recommender systems~\citep{ying2018graph}.
 Recently, newer convolutional GNN structures \citep{wu2019comprehensive} have drastically improved the performance of GNNs by employing various techniques, including renormalization \citep{kipf2016semi}, attention \citep{velivckovic2017graph}, and simpler activation \citep{pmlrv97wu19e}.

The aforemetioned modern convolutional GNN models capture relation information up to $T$-hops away by performing $T$ iterations of graph convolutional aggregation. Such information gathering procedure is similar to forward-feeding schemes in popular deep learning models, such as multi-layer perceptron and convolutional neural networks. However, despite their simplicity, these computation strategies cannot discover the dependency with a range longer than $T$-hops away from any given node.

One approach tackling this problem is to develop recurrent GNNs that iterate graph convolutional aggregation until convergence, without any {\em a priori} limitation on the number of hops. This idea arises in many traditional graph metrics, including eigenvector centrality~\citep{newman2018networks} and PageRank~\citep{page1999pagerank}, where the metrics are implicitly defined by some fixed-point equation. Intuitively, the long-range dependency 
can be better captured by iterating the information passing procedure for an infinite number of times until convergence. Pioneered by~\citep{gori2005new}, new recurrent GNNs leverage partial training \citep{gallicchio2010graph, gallicchio2019fast} and approximation \citep{dai2018learning} to improve performance. With shared weights, these methods avoid exploding memory issues and achieve accuracies competitive with convolutional counterparts in certain cases.

While these methods offer an alternative to the popular convolutional GNN models with added benefits for certain problems, there are still significant limitations in evaluation and training for recurrent GNN models. Conservative convergence conditions and sophisticated training procedures have limited the use of these methods in practice, and outweighed the performance benefits of capturing the long-range dependency. In addition, most of these methods cannot leverage multi-graph information or adapt to heterogeneous network settings, as prevalent in social networks as well as bio-chemical graphs \citep{wan2019neodti}. 

\paragraph{Paper contributions.}
In this work, we present the \textbf{I}mplicit \textbf{G}raph \textbf{N}eural \textbf{N}etwork (IGNN) framework to address the problem of evaluation and training for recurrent GNNs. We first analyze graph neural networks through a rigorous mathematical framework based on the Perron-Frobenius theory \citep{berman1994nonnegative}, in order to establish general well-posedness conditions for convergence. We show that most existing analyses are special cases of our result. As for training, we propose a novel projected gradient method to efficiently train the IGNN, where we leverage implicit differentiation methods to obtain the exact gradient, and use projection on a tractable convex set to guarantee well-posedness. We show that previous gradient methods for recurrent graph neural networks can be interpreted as an approximation to IGNN. Further, we extend IGNN to heterogeneous network settings. 
Finally, we conduct comprehensive comparisons with existing methods, and demonstrate that our method effectively captures long-range dependencies and outperforms the state-of-the-art GNN models on a wide range of tasks.

\paragraph{Paper outline.} In Section \ref{sec:related_work}, we give an overview of related work on GNN and implicit models. In Section \ref{sec:notation}, we introduce the background and notations for this paper. Section \ref{sec:ignn} discusses the IGNN framework together with its well-posedness and training under both ordinary and heterogeneous settings. Section \ref{sec:numerical} empirically compares IGNN with modern GNN methods.

\section{Related Work} \label{sec:related_work}
\paragraph{GNN models.}
Pioneered by \citep{gori2005new}, GNN models have gained influence for graph-related tasks. Led by GCN \citep{kipf2016semi}, convolutional GNN models ~\citep{velivckovic2017graph,hamilton2017inductive,pmlrv97wu19e,jin2019power,chang2020spectral} involve a finite number of modified aggregation steps with different weight parameters. On the other hand, recurrent GNN models \citep{gori2005new} use the same parameters for each aggregation step and potentially enable infinite steps.  
\citep{li2015gated} combines recurrent GNN with recurrent neural network structures. Methods such as Fast and Deep Graph Neural Network (FDGNN) \citep{gallicchio2010graph,gallicchio2019fast} use untrained recurrent GNN models with novel initialization to its aggregation step for graph classification. While the Stochastic Steady-State Embedding (SSE) method \citep{dai2018learning} uses an efficient approximated training and evaluation procedure for node classification. 
Recently, global method Geom-GCN~\citep{pei2020geom} employs additional embedding approaches to capture global information. However, Geom-GCN~\citep{pei2020geom} also belongs to convolutional-based GNNs, which struggle to capture very long range dependency due to the finite iterations they take.

\paragraph{Implicit Models.}
Implicit models are emerging structures in deep learning where the outputs are determined implicitly by a solution of some underlying sub-problem. Recent works \citep{bai2019deep} demonstrate the potential of implicit models in sequence modeling, physical engine~\citep{NIPS2018_7948} and many others \citep{chen2018neural, amos2018differentiable}.
\citep{elghaoui2020} proposes a general implicit framework with the prediction rule based on the solution of a fixed-point equilibrium equation and discusses the well-posedness of the implicit prediction rule. 

\paragraph{Oversmoothing.}
To catch the long-range dependency, another intuitive approach is to construct deeper convolutional GNNs by stacking more layers. However, ~\citep{li2018deeper} found that the learned node embeddings become indistinguishable as the convolutional GNNs get deeper. This phenomenon is called \textit{over-smoothing}. Since then, a line of empirical~\citep{li2018deeper,chen2020simple,rong2020dropedge} and theoretical~\citep{oono2020graph,zhao2020pairnorm} works follows on the \textit{over-smoothing} phenomenon. Unlike convolutional GNNs, IGNN adopts a different approach for long-range dependency based on recurrent GNNs and doesn't seem to suffer performance degradation as much even though it could be viewed as an infinite-layer GNN. See appendix \ref{app:oversmooth} for details.


\section{Preliminaries} \label{sec:notation}
Graph neural networks take input data in the form of graphs. A graph is represented by $G = (V, E)$ where $V$ is the set of $n:=|V|$ nodes (or vertices) and $E\subseteq V \times V$ is the set of edges. In practice, we construct an adjacency matrix $A\in \reals^{n\times n}$ to represent the graph $G$: for any two nodes $i, j \in V$, if $(i, j) \in E$, then $A_{ij} = 1$; otherwise, $A_{ij} = 0$. Some data sets provide additional information about the nodes in the form of a feature matrix $U\in \reals^{p\times n}$, in which the feature vector for node $i$ is given by $u_i\in \reals^p$. 
When no additional feature information from nodes is provided, the data sets would require learning a feature matrix $U$ separately in practice. 

Given graph data, graph models produce a prediction $\hat{Y}$ to match the true label $Y$ whose shape depends on the task. GNN models are effective in graph-structured data because they involve trainable aggregation steps that pass the information from each node to its neighboring nodes and then apply nonlinear activation. The aggregation step at iteration $t$ can be written as follows:
\begin{equation} \label{eq:matrix_agg}
    X^{(t+1)} = \phi(W^{(t)}X^{(t)}A + \Omega^{(t)}U),
\end{equation}
where $X^{(t)} \in \reals^{m\times n}$ stacks the state vectors of nodes in time step $t$ into a matrix, in which the state vector for node $i$ is denoted as $x^{(t)}\in \reals^m$; $W^{(t)}$ and $\Omega^{(t)}$ are trainable weight matrices; $\phi$ is an activation function. The state vectors $X^{(T)}$ at final iteration can be used as the representation for nodes that combine input features and graph spatial information. The prediction from the GNN models is given by $\hat{Y} = f_\Theta(X^{(T)})$, where $f_\Theta$ is some trainable function parameterized by $\Theta$. In practice, a linear $f_\Theta$ is often satisfactory.

Modern GNN approaches adopt different forms of graph convolution aggregation (\ref{eq:matrix_agg}). Convolutional GNNs~\citep{wu2019comprehensive} iterate (\ref{eq:matrix_agg}) with $\Omega = 0$ and set $X^{(0)} = U$. Some works temper with the adjacency matrix using renormalization \citep{kipf2016semi} or attention \citep{velivckovic2017graph}). While recurrent GNNs use explicit input from features at each step with tied weights $W$ and $\Omega$, some 
methods replace the term $\Omega U$ with $\Omega_1 U A + \Omega_2U$, in order to account for feature information from neighboring nodes \citep{dai2018learning}. Our framework adopts a similar recurrent graph convolutional aggregation idea. 

A heterogeneous network
 is an extended type of graph that contains different types of relations between nodes instead of only one type of edge. We continue to use $G=(V, \mathcal{E})$ to represent a heterogeneous network with the node set $V$ and the edge set $\mathcal{E} \subseteq V \times V \times R$, where $R$ is a set of $N:=|R|$ relation types. Similarly, we define the adjacency matrices $A_i$, where $A_i$ is the adjacency matrix for relation type $i\in R$. Some heterogeneous networks also have relation-specific feature matrices $U_i$.

\textbf{Notation.}
For a matrix $V\in \reals^{p \times q}$, $|V|$ denotes its absolute value (\ie\; $|V|_{ij} = |V_{ij}|$). The infinity norm, or the max-row-sum norm, writes $\|V\|_\infty$. The 1-norm, or the max-column-sum norm, is denoted as $\|V\|_1 = \|V^\top\|_\infty$. The 2-norm is shown as $\|V\|$ or $\|V\|_2$. We use $\otimes$ to represent the Kronecker product, $\langle\cdot, \cdot\rangle$ to represent inner product and use $\odot$ to represent component-wise multiplication between two matrices of the same shape. For a $p \times q$ matrix $V$, $\vecc(V)\in \reals^{pq}$ represents the vectorized form of $V$, obtained by stacking its columns (See Appendix \ref{app:kron} for details). According to the Perron-Frobenius theory \citep{berman1994nonnegative}, every squared non-negative matrix $M$ has a real non-negative eigenvalue that gives the largest modulus among all eigenvalues of $M$. This non-negative eigenvalue of $M$ is called the \emph{Perron-Frobenius (PF) eigenvalue} and denoted by $\lpf(M)$ throughout the paper.

\section{Implicit Graph Neural Networks} \label{sec:ignn}
We now introduce a framework for graph neural networks called \textbf{I}mplicit \textbf{G}raph \textbf{N}eural \textbf{N}etworks (IGNN), which obtains a node representation through the fixed-point solution of a non-linear ``equilibrium'' equation. 
The IGNN model is formally described by
\begin{subequations}\label{eq:ignn}
\begin{align}
  \hat{Y} &= f_\Theta(X) \label{eq:ignn_rep} ,\\
  X &= \phi(WXA+b_\Omega (U)) \label{eq:ignn_equ} .
\end{align}
\end{subequations}
In equation (\ref{eq:ignn}), the input feature matrix $U\in\reals^{p \times n}$ is passed through some affine transformation $b_\Omega(\cdot)$ parametrized by $\Omega$ (\ie\; a linear transformation possibly offset by some bias). The representation, given as the ``internal state'' $X \in \reals^{m \times n}$ in the rest of the paper, is obtained as the fixed-point solution of the equilibrium equation (\ref{eq:ignn_equ}), where $\phi$ preserves the same shape of input and output. The prediction rule (\ref{eq:ignn_rep}) computes the prediction $\hat{Y}$ by feeding the state $X$ through the output function $f_\Theta$. In practice, a linear map $f_\Theta(X) = \Theta X$ may be satisfactory.

Unlike most existing methods that iterate (\ref{eq:matrix_agg}) for a finite number of steps, an IGNN seeks the fixed point of equation (\ref{eq:ignn_equ}) that is trained to give the desired representation for the task. Evaluation of fixed point can be regarded as iterating (\ref{eq:matrix_agg}) for an infinite number of times to achieve a steady state. Thus, the final representation potentially contains information from all neighbors in the graph. In practice, this gives a better performance over the finite iterating variants by capturing the long-range dependency in the graph. Another notable benefit of the framework is that it is memory-efficient in the sense that it only maintains one current state $X$ without other intermediate representations.

Despite its notational simplicity, the IGNN model covers a wide range of variants, including their multi-layer formulations by stacking multiple equilibrium equations similar to (\ref{eq:ignn_equ}). The SSE \citep{dai2018learning} and FDGNN \citep{gallicchio2019fast} models also fit within the IGNN formulation. We elaborate on this aspect in Appendix \ref{app:example_ignn}.

IGNN models can generalize to heterogeneous networks with different adjacency matrices $A_i$ and input features $U_i$ for different relations. In that case, we have the parameters $W_i$ and $\Omega_i$ for each relation type $i\in R$ to capture the heterogenerity of the graph. A new equilibrium equation (\ref{eq:ignn_h_equ}) is used:
\begin{equation} \label{eq:ignn_h_equ}
  X = \phi \left( \sum_i (W_iXA_i+b_{\Omega_i}(U_i))\right) .
\end{equation}
In general, there may not exist a unique solution for the equilibrium equation (\ref{eq:ignn_equ}) and (\ref{eq:ignn_h_equ}). Thus, the notion of well-posedness comes into play.

\subsection{Well-posedness of IGNNs} \label{sub:wellposedness}
For the IGNN model to produce a valid representation, we need to obtain some \textit{unique} internal state $X(U)$ given any input $U$ from equation (\ref{eq:ignn_equ}) for ordinary graph settings or equation (\ref{eq:ignn_h_equ}) for heterogeneous network settings. However, the equilibrium equation (\ref{eq:ignn_equ}) and (\ref{eq:ignn_h_equ}) can have no well-defined solution $X$ given some input $U$. We give a simple example in the scalar setting in Appendix~\ref{app:wp}, where the solution to the equilibrium equation (\ref{eq:ignn_equ}) does not even exist. 

In order to ensure the existence and uniqueness of the solution to equation (\ref{eq:ignn_equ}) and (\ref{eq:ignn_h_equ}), we define the notion of \emph{well-posedness} for equilibrium equations with activation $\phi$ for both ordinary graphs and hetergeneous networks. This notion has been introduced in \citep{elghaoui2020} for ordinary implicit models.
\begin{definition} [Well-posedness on ordinary graphs]
    The tuple $(W, A)$ of the weight matrix $W\in\reals^{m\times m}$ and the adjacency matrix $A\in\reals^{n\times n}$ is said to be well-posed for $\phi$ if for any $B\in\reals^{m\times n}$, the solution $X\in\reals^{m\times n}$ of the following equation
    \begin{equation} \label{eq:dummy_equ}
        X = \phi (WXA+B)
    \end{equation}
     exists and is unique.
\end{definition}

\begin{definition} [Well-posedness on heterogeneous networks]
    The tuple $(W_i, A_i, i=1, \dots, N)$ of the weight matrices $W_i\in\reals^{m\times m}$ and the adjacency matrices $A_i\in\reals^{n\times n}$ is said to be well-posed for $\phi$ if for any $B_i\in\reals^{m\times n}$, the solution $X\in\reals^{m\times n}$ of the following equation 
    \begin{equation} \label{eq:dummy_equ_h}
        X = \phi \left( \sum_{i=1}^N (W_iXA_i+B_i)\right)
    \end{equation}
    exists and is unique.
\end{definition}

We first develop sufficient conditions for the well-posedness property to hold on ordinary graph settings with a single edge type. The idea is to limit the structure of $W$ and $A$ together to ensure well-posedness for a set of activation $\phi$. 

In the following analysis, we assume that $\phi$ is component-wise non-expansive, which we refer to as the component-wise non-expansive (CONE) property. Most activation functions in deep learning satisfy the CONE property (\eg\; Sigmoid, tanh, ReLU, Leaky ReLU, \etc). For simplicity, we assume that $\phi$ is differentiable. 

We can now establish the following sufficient condition on $(W,A)$ for our model with a CONE activation to be well-posed. Our result hinges on the notion of Perron-Frobenius (PF) eigenvalue $\lpf(M)$ for a non-negative matrix $M$, as well as the notion of Kronecker product $A\otimes B \in \reals^{pm\times qn}$ between two matrices $A\in\reals^{m\times n}$ and $B\in\reals^{p\times q}$.
See Appendix \ref{app:kron} for details.

\begin{theorem} [PF sufficient condition for well-posedness on ordinary graphs] \label{thm:wp}
Assume that $\phi$ is a component-wise non-expansive (CONE) activation map. Then, $(W, A)$ is well-posed for any such $\phi$ if $\lpf(|A^\top \otimes W|) < 1$. Moreover, the solution $X$ of equation (\ref{eq:dummy_equ}) can be obtained by iterating equation (\ref{eq:dummy_equ}).
\end{theorem}
\begin{proof}
Recall that for any three matrices $A,W,X$ of compatible sizes, we have $(A^\top \otimes W) \vecc(X) = \vecc(WXA)$ \citep{schacke2004kronecker}.
Showing equation (\ref{eq:dummy_equ}) has an unique solution is equivalent to showing that the following ``vectorized'' equation has a unique solution:
$$
\vecc(X) = \phi(A^\top \otimes W \vecc(X) + \vecc(B))
$$
It follows directly from Lemma \ref{lem:wp_orig} that if $\lpf(|A^\top \otimes W|) = \lpf(A)\lpf(|W|) < 1$, then the above equation has unique solution that can be obtained by iterating the equation.
\end{proof}

We find Theorem \ref{thm:wp} so general that many familiar and interesting results will follow from it, as discussed in the following remarks. Detailed explanations can be found in Appendix \ref{app:wp}.

\begin{remark}[Contraction sufficient condition for well-posedness \citep{gori2005new}] \label{rem:4.2}
For any component-wise non-expansive (CONE) $\phi$, if $\mathcal{A}(X) = \phi(WXA+B)$ is a contraction of $X$ (w.r.t. vectorized norms), then $(W,A)$ is well-posed for $\phi$.
\end{remark}

\begin{remark}[Well-posedness for directed acyclic graph] \label{rem:4.3}
For a directed acyclic graph (DAG), let $A$ be its adjacency matrix. For any real squared $W$, it holds that $(W,A)$ is well-posed for every CONE activation map. Note that $\mathcal{A}(X) = \phi(WXA+B)$ need not be a contraction of $X$.
\end{remark}

\begin{remark}[Sufficient well-posedness condition for k-regular graph \citep{gallicchio2019fast}] \label{rem:4.4}
For a k-regular graph, let $A$ be its adjacency matrix. $(W,A)$ is well-posed for every CONE activation map if $k\|W\|_2<1$.
\end{remark}

A similar sufficient condition for well-posedness holds for heterogeneous networks.
\begin{theorem}[PF sufficient condition for well-posedness on heterogeneous networks] \label{thm:wp_h}
Assume that $\phi$ is some component-wise non-expansive (CONE) activation map. Then, $(W_i, A_i,\; i = 1, \dots, N)$ is well-posed for any such $\phi$ if $\lpf\left(\sum_{i=1}^N |A_i^\top \otimes W_i|\right) < 1$. Moreover, the solution $X$ of equation (\ref{eq:dummy_equ_h}) can be obtained by iterating equation (\ref{eq:dummy_equ_h}).
\end{theorem}
We give a complete proof in Appendix \ref{app:wp}. Sufficient conditions in Theorems \ref{thm:wp} and \ref{thm:wp_h} guarantee convergence when iterating aggregation step to evaluate state $X$. Furthermore, these procedures enjoy exponential convergence in practice.

\subsection{Tractable Well-posedness Condition for Training} \label{sub:tractable}
At training time, however, it is difficult in general to ensure satisfaction of the PF sufficient condition $\lpf(|W|)\lpf(A)<1$, because $\lpf(|W|)$ is non-convex in $W$. To alleviate the problem, we give a numerically tractable convex condition for well-posedness that can be enforced at training time efficiently through projection. Instead of using $\lpf(|W|) < \lpf(A)^{-1}$, we enforce the stricter condition $\|W\|_\infty < \lpf(A)^{-1}$, which guarantees the former inequality by $\lpf(|W|) \le \|W\|_\infty$. Although $\|W\|_\infty < \lpf(A)^{-1}$ is a stricter condition, we show in the following theorem that it is equivalent to the PF condition for positively homogeneous activation functions, (\ie\; $\phi(\alpha x) = \alpha \phi(x)$ for any $\alpha \ge 0$ and $x$), in the sense that one can use the former condition at training without loss of generality.

\begin{theorem} [Rescaled IGNN] \label{thm:rescale}
Assume that $\phi$ is CONE and positively homogeneous. For an IGNN ($f_\Theta, W, A, b_\Omega, \phi$) where $(W, A)$ satisfies the PF sufficient condition for well-posedness, namely $\lpf(|W|) < \lpf(A)^{-1}$, there exists a linearly-rescaled equivalent IGNN ($\Tilde{f_\Theta}, W', A, \Tilde{b_\Omega}, \phi$) with $\|W'\|_{\infty} < \lpf(A)^{-1}$ that gives the same output $\hat{Y}$ as the original IGNN for any input U.
\end{theorem}

The proof is given in Appendix \ref{app:wp}. The above-mentioned condition can be enforced by selecting a $\kappa \in [0, 1)$ and projecting the updated $W$ onto the convex constraint set $\mathcal{C} = \{W:\|W\|_\infty \le \kappa/\lpf(A) \}$. 

For heterogeneous network settings, we recall the following:
\begin{remark} \label{rem:4.5}
For any non-negative adjacency matrix $A$ and arbitrary real parameter matrix $W$, it holds that $\|A^\top \otimes W\|_\infty = \|A^\top\|_\infty \|W\|_\infty = \|A\|_1 \|W\|_\infty$. 
\end{remark}

Similar to the difficulty faced in the ordinary graph settings, ensuring the PF sufficient condition on heterogeneous networks is hard in general. We propose to enforce the following tractable condition that is convex in $W_i$'s: $\sum_{i=1}^N \|A_i\|_1 \|W_i\|_\infty \le \kappa < 1$, $\kappa\in[0, 1)$. Note that this condition implies $\left\|\sum_{i=1}^N A_i^\top \otimes W_i\right\|_\infty \le \kappa$, and thus $\lpf\left(\sum_{i=1}^N |A_i^\top \otimes W_i|\right) \le \kappa < 1$. The PF sufficient condition for well-posedness on heterogeneous networks is then guaranteed.


\subsection{Training of IGNN}
We start by giving the training problem (\ref{eq:training}), where a loss $\mathcal{L}(Y, \hat{Y})$ is minimized to match $\hat{Y}$ to $Y$ and yet the tractable condition $\|W\|_\infty \le \kappa/\lpf(A)$ for well-posedness is enforced with $\kappa\in[0, 1)$:

\begin{equation} \label{eq:training}
    \min_{\Theta, W,\Omega} \;\mathcal{L}(Y, f_\Theta(X)) \; : \;\; X = \phi(WXA+b_\Omega (U)), \;\; \|W\|_\infty \le \kappa /\lpf(A).
\end{equation}

The problem can be solved by projected gradient descent (involving a projection to the well-posedness condition following a gradient step), where the gradient is obtained through an implicit differentiation scheme. From the chain rule, one can easily obtain $\nabla_\Theta \mathcal{L}$ for the parameter of $f_\Theta$ and $\nabla_X \mathcal{L}$ for the internal state $X$. In addition, we can write the gradient with respect to scalar $q \in W \cup \Omega$ as follows:
\begin{equation} \label{eq:q_grad}
    \nabla_{q} \mathcal{L} = \left\langle \frac{\partial \left(WXA + b_\Omega (U)\right)}{\partial q}, \nabla_Z \mathcal{L}\right\rangle,
\end{equation}
where $Z = WXA + b_\Omega (U)$ \emph{assuming fixed $X$} (see Appendix \ref{app:gradient}). Here, $\nabla_Z \mathcal{L}$ 
is given as a solution to the equilibrium equation
\begin{equation} \label{eq:gradient_z}
    \nabla_Z \mathcal{L} = D \odot \left(W^\top \nabla_Z \mathcal{L}\;A^\top + \nabla_X \mathcal{L}\right),
\end{equation}
where $D = \phi'(WXA+b_\Omega(U))$ and $\phi'(z) = d\phi(z)/d z$ refers to the element-wise derivative of the CONE map $\phi$. Since $\phi$ is non-expansive, it is 1-Lipschitz (\ie\; the absolute value of $d\phi(z)/d z$ is not greater than 1), the equilibrium equation (\ref{eq:gradient_z}) for gradient $\nabla_Z \mathcal{L}$ admits a \textit{unique} solution by iterating (\ref{eq:gradient_z}) to convergence, if $(W,A)$ is well-posed for any CONE activation $\phi$. (Note that $D\odot (\cdot)$ can be seen as a CONE map with each entry of $D$ having absolute value less than or equal to 1.) Again, $\nabla_Z \mathcal{L}$ can be efficiently obtained due to exponential convergence when iterating (\ref{eq:gradient_z}) in practice.

Once $\nabla_Z \mathcal{L}$ is obtained, we can use the chain rule (via autograd software) to easily compute $\nabla_W \mathcal{L}$, $\nabla_\Omega \mathcal{L}$, and possibly $\nabla_U \mathcal{L}$ when input $U$ requires gradients (\eg\;in cases of features learning or multi-layer formulation). The deriviation has a deep connection to the Implicit Function Theorm. See Appendix \ref{app:gradient} for details.  

Due to the norm constraint introduced for well-posedness, each update to $W$ requires a projection step (See Section \ref{sub:wellposedness}). The new $W$ is given by $W^+ = \pi_{\mathcal{C}}(W) := \argmin_{\|M\|_\infty \le \kappa/\lpf(A)} \|M-W\|_F^2$, where $\pi_{\mathcal{C}}$ is the projection back onto $\mathcal{C} = \{\|W\|_\infty \le \kappa/\lpf(A)\}$. The projection is decomposible across the rows of $W$. Each sub-problem will be a projection onto an $\mathcal{L}_1$-ball for which efficient methods exist \citep{duchi2008efficient}. A similar projected gradient descent training scheme for heterogeneous network settings is detailed in Appendix \ref{app:gradient}. Note that the gradient method in SSE \citep{dai2018learning} uses a first-order approximated solution to equation (\ref{eq:gradient_z}). FDGNN \citep{gallicchio2019fast} only updates $\Theta$ at training using gradient descent.


\section{Numerical Experiment} \label{sec:numerical} 
In this section, we demonstrate the capability of IGNN on effectively learning a representation that captures the long-range dependency and therefore offers the state-of-the-art performance on both synthetic and real-world data sets.
More specifically, we test IGNN against a selected set of baselines on 6 node classification data sets (Chains, PPI, AMAZON, ACM, IMDB, DBLP) and 5 graph classification data sets (MUTAG, PTC, COX2, PROTEINS, NC11), where Chains is a synthetic data set; PPI and AMAZON are multi-label classification data sets; ACM, IMDB and DBLP are based on heterogeneous networks. We inherit the same experimental settings and reuse the 
results of baselines from literatures in some of the data sets. The test set performance is reported. Detailed description of the data sets, our preprocessing procedure, hyper-parameters, and other information of experiments can be found in Appendix \ref{app:more experiments}.

\paragraph{Synthetic Chains Data Set.}

To evaluate GNN's capability for capturing the underlying long-range dependency in graphs, we create the Chains data set where the goal is to classify nodes in a chain of length $l$. The information of the class is only sparsely provided as the feature in an end node. We use a small training set, validation set, and test set with only 20, 100, and 200 nodes, respectively. For simplicity, we only consider the binary classification task. Four representative baselines are implemented and compared. We show in Figure~\ref{fig:chains} that IGNN and SSE~\citep{dai2018learning} both capture the long-range dependency with IGNN offering 
a better performance for longer chains, while finite-iterating convolutional GNNs with $T=2$, including GCN~\citep{kipf2016semi}, SGC~\citep{pmlrv97wu19e} and GAT~\citep{velivckovic2017graph}, fail to give meaningful predictions when the chains become longer. However, selecting a larger $T$ for convolutional GNNs does not seem to help in this case of limited training data. We further discuss this aspect in Appendix \ref{app:more experiments}. 

\begin{figure}[htbp]
	\centering
	\begin{minipage}{0.45\textwidth}
		\centering
		\caption{Micro-$F_{1}$ (\%) performance with respect to the length of the chains.}
		\label{fig:chains}
		\includegraphics[width=0.95\textwidth]{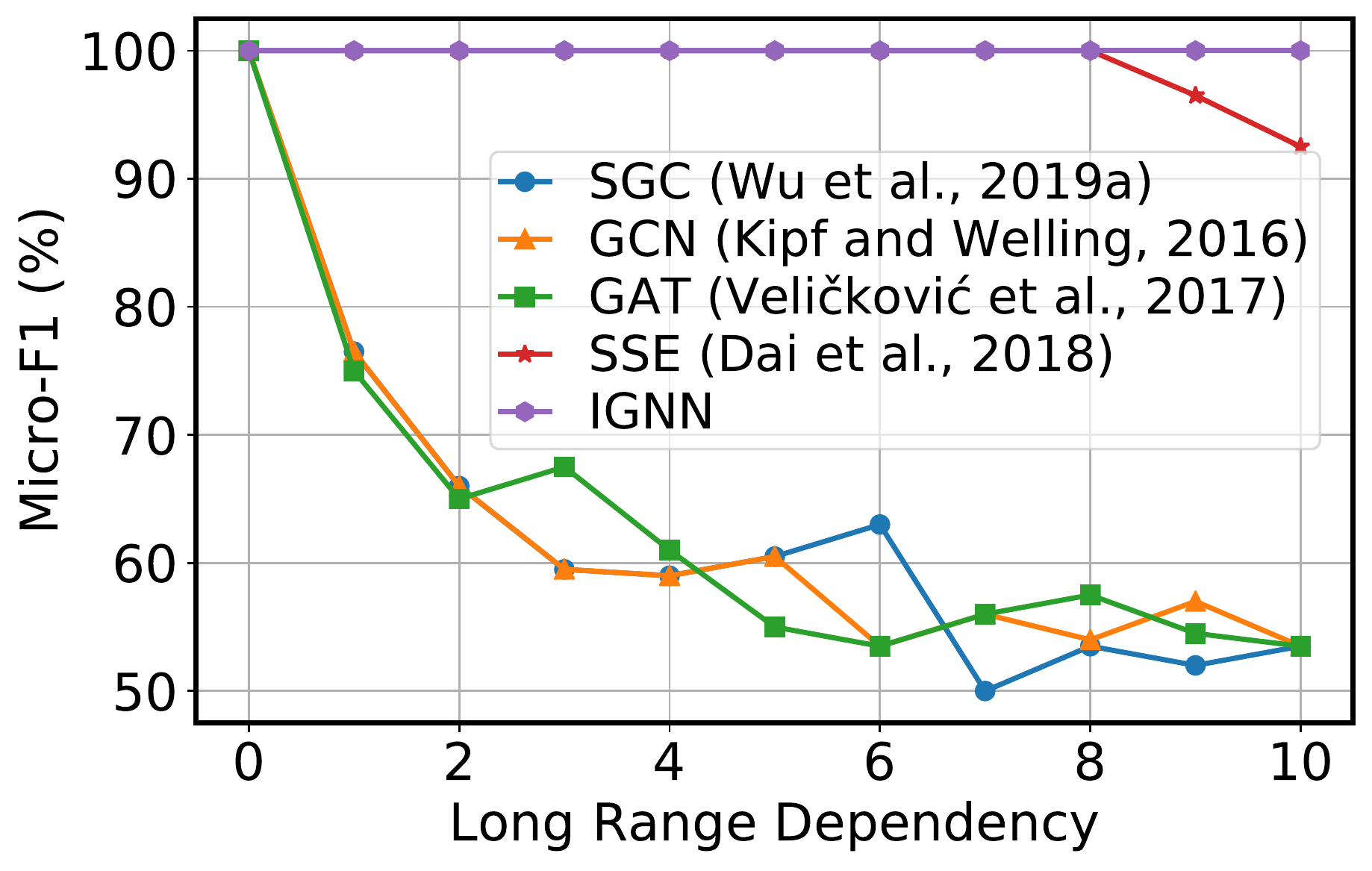}
	\end{minipage}
	\hfill
	\begin{minipage}{0.48\textwidth}
	    \vspace{-5mm}
	    \captionof{table}{Multi-label node classification Micro-$F_{1}$ (\%) performance on PPI data set. \label{tab:PPI}}
        \centering
        \resizebox{0.99\textwidth}{!}{%
        \begin{tabular}{lcc@{}}
        \toprule
        Method & Micro-$F_{1}$ /\% \\
        \midrule
        Multi-Layer Perceptron & $46.2$ \\ 
        GCN~\citep{kipf2016semi} & $59.2$ \\ 
        GraphSAGE~\citep{hamilton2017inductive} & $78.6$  \\
        SSE~\citep{dai2018learning} & $83.6$  \\
        GAT~\citep{velivckovic2017graph} & $97.3$ \\ 
        \midrule
        IGNN & $\mathbf{97.6}$ \\
        \bottomrule
        \end{tabular}
        }
	\end{minipage}
\end{figure}

\paragraph{Node Classification.}

The popular benchmark data set Protein-Protein Interaction (PPI) models the interactions between proteins using a graph, with nodes being proteins and edges being interactions. 
Each protein can have at most 121 labels and be associated with additional 50-dimensional features. The train/valid/test split is consistent with GraphSage~\citep{hamilton2017inductive}. We report the micro-averaged $F_{1}$ score of a multi-layer IGNN against other popular baseline models. The results can be found in Table~\ref{tab:PPI}. By capturing the underlying long-range dependency between proteins, the IGNN achieves the best performance compared to other baselines.

\begin{figure}[htbp]
    \caption{Micro/Macro-$F_{1}$  (\%) performance on the multi-label node classification task with Amazon product co-purchasing network data set. \label{fig:amazon}}
	\centering
	\begin{minipage}{0.45\textwidth}
		\centering
		\includegraphics[width=0.95\textwidth]{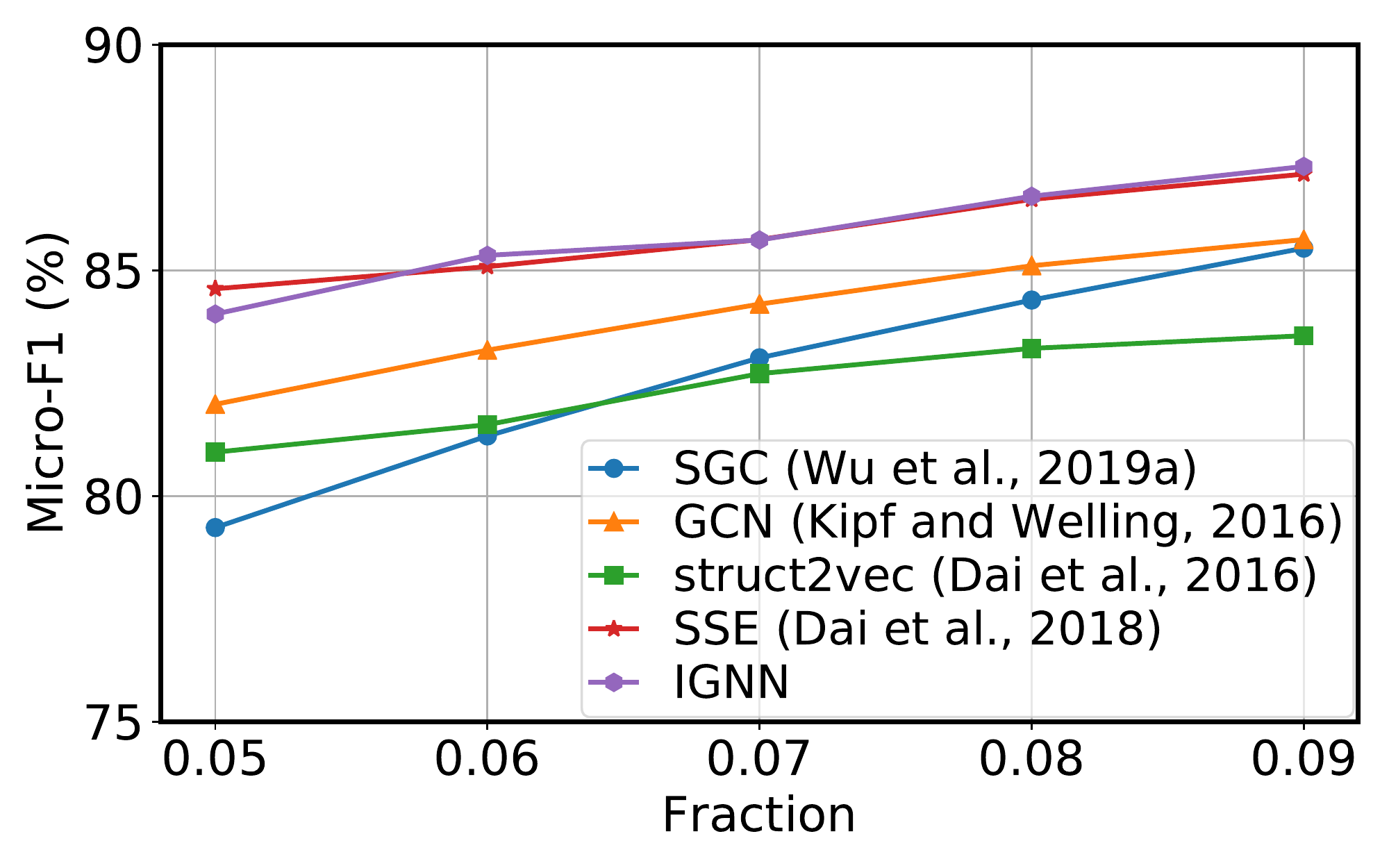}
	\end{minipage}
	\hfill
	\begin{minipage}{0.45\textwidth}
		\centering
		\includegraphics[width=0.95\textwidth]{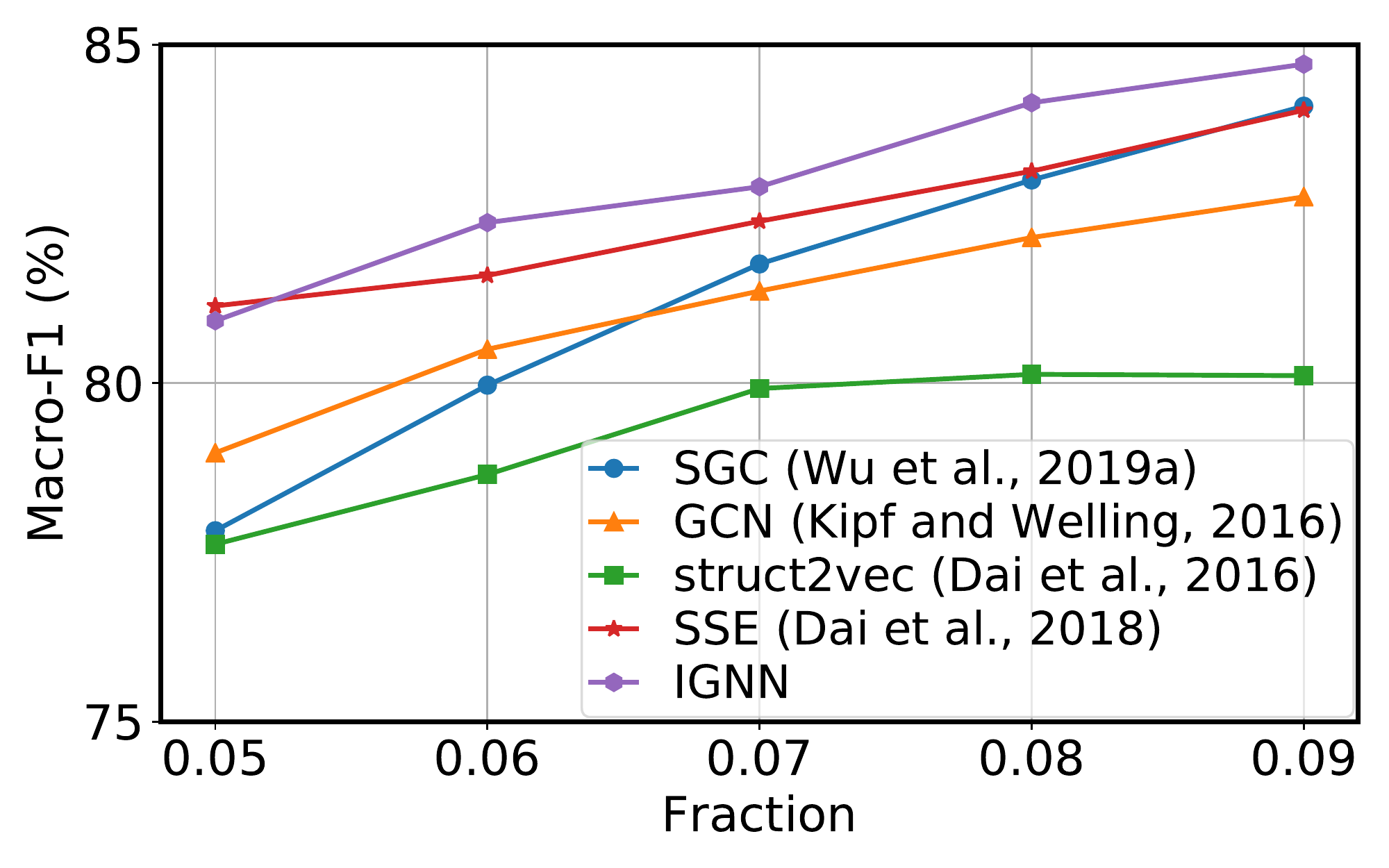}
	\end{minipage}
	\vspace{-3mm}
\end{figure}


To further manifest the scalability of IGNN towards larger graphs, we conduct experiments on a large multi-label node classification data set, namely the Amazon product co-purchasing network data set~\citep{yang2015defining} \footnote{http://snap.stanford.edu/data/\#amazon}. The data set renders products as nodes and co-purchases as edges but provides no input features. 58 product types with more than 5,000 products are selected from a total of 75,149 product types. While holding out 10\% of the total nodes as test set, we vary the training set fraction from 5\% to 9\% to be consistent with~\citep{dai2018learning}. The data set come with no input feature vectors and thus require feature learning at training.
Both Micro-$F_{1}$ and Macro-$F_{1}$ are reported on the held-out test set, where we compare IGNN with a set of baselines consistent with those in the synthetic data set. However, we use struct2vec~\citep{dai2016discriminative} as an alternative to GAT since GAT faces a severe out-of-memory issue in this task.

As shown in Figure~\ref{fig:amazon}, IGNN again outperforms the baselines in most cases, especially when the amount of supervision grows. When more labels are available, more high-quality feature vectors of the nodes are learned and this enables the discovery of more long-range dependency. This phenomenon is aligned with our observation that IGNN achieves a better performance when there is more long-range dependency in the underlying graph.


\paragraph{Graph Classification.} 
Aside from node classification, we test IGNN on graph classification tasks. A total of 5 bioinformatics benchmarks are chosen: MUTAG, PTC, COX2, NCI1 and PROTEINS~\citep{yanardag2015deep}. See details of data sets in Appendix \ref{app:more experiments}. Under the graph classification setting, we compare a multi-layer IGNN with a comprehensive set of baselines, including a variety of GNNs and a number of graph kernels. Following identical settings as \citep{yanardag2015deep,xu2018powerful}, 10-fold cross-validation with LIB-SVM~\citep{chang2011libsvm} is conducted. The average prediction accuracy and standard deviations are reported in Table~\ref{tab:graph classification}. In this experiment, IGNN achieves the best performance in 4 out of 5 experiments given the competitive baselines. Such performance further validates IGNN's success in learning converging aggregation steps that capture long-range dependencies when generalized to unseen testing graphs. 

\begin{table*}[tbh]
  \caption{Graph classification accuracy (\%).
  Results are averaged (and std are computed) on the outer 10 folds.}
  \vspace{1mm}
  \label{tab:graph classification}
  \centering
  \small
  \resizebox{0.9\textwidth}{!}{%
  \begin{tabular}{cccccc}
    \toprule
    
    Data sets & MUTAG & PTC & COX2 & PROTEINS & NCI1\\
    \# graphs & 188 & 344 & 467 & 1113 & 4110 \\
    Avg \# nodes & 17.9 & 25.5 & 41.2 & 39.1 & 29.8 \\
    \midrule
    DGCNN~\citep{zhang2018end}
            & $85.8$ & $58.6$ &
              $-$                  & $75.5$ &
              $74.4$\\
    DCNN~\citep{atwood2016diffusion}
            & $67.0$ &
            $56.6$ &
              $-$                  & 
              $61.3$ &
              $62.6$\\

    GK~\citep{shervashidze2009efficient}
            & $81.4 \pm 1.7$ & $55.7 \pm 0.5$ &
              $-$                  & $71.4 \pm 0.3 $ &
              $62.5 \pm 0.3$\\

    RW~\citep{gartner2003graph}
            & $79.2 \pm 2.1 $ & $55.9 \pm 0.3$ &
              $-$   & $59.6 \pm 0.1$ &
              $-$\\
    PK~\citep{neumann2016propagation}
            & $76.0 \pm 2.7$ & $59.5 \pm 2.4$ &
            $81.0 \pm 0.2$  & $73.7 \pm 0.7$ &
              $82.5 \pm 0.5$\\
    WL~\citep{shervashidze2011weisfeiler}
            & $84.1 \pm 1.9$ & $58.0 \pm 2.5$ &
            $83.2 \pm 0.2$  & $74.7 \pm 0.5$ &
              $\mathbf{84.5 \pm 0.5}$\\
    FDGNN~\citep{gallicchio2019fast}  & $88.5 \pm 3.8$  & 
              $63.4 \pm 5.4$ &
              $83.3 \pm 2.9$ &  
              $76.8 \pm 2.9$ &  
              $77.8 \pm 1.6$ \\
    GCN~\citep{kipf2016semi}  & $85.6 \pm 5.8$  & 
              $64.2 \pm 4.3$ &
              $-$ &  
              $76.0 \pm 3.2$ &  
              $80.2 \pm 2.0$ \\
    GIN~\citep{xu2018powerful}  & $89.0 \pm 6.0$  & 
              $63.7 \pm 8.2$ &
              $-$ &  
              $75.9 \pm 3.8$ &  
              $82.7 \pm 1.6$ \\
    \midrule
    IGNN  & $\mathbf{89.3\pm6.7}$  & 
              $\mathbf{70.1\pm5.6}$ &
              $\mathbf{86.9\pm4.0}$ &  
              $\mathbf{77.7\pm3.4}$ &  
              $80.5\pm1.9$ \\
    \bottomrule
  \end{tabular}
  }
\end{table*}

\paragraph{Heterogeneous Networks.}
Following our theoretical analysis on heterogeneous networks, we investigate how IGNN takes advantage of heterogeneity on node classification tasks. Three benchmarks based on heterogeneous network are chosen, \emph{i.e.}, ACM, IMDB and DBLP \citep{wang2019heterogeneous,park2019unsupervised}. More information regarding the heterogeneous network data sets can be found in Appendix \ref{app:more experiments}. Table~\ref{tab:Heterogeneous Networks} compares IGNN against a set of state-of-the-art GNN baselines for heterogeneous networks. The heterogeneous variant of IGNN continues to offer a competitive performance on all 3 data sets where IGNN gives the best performance in ACM and IMDB data sets. While on DBLP, IGNN underperforms DMGI but still outperforms other baselines by large margin. Good performance on heterogeneous networks demonstrates 
the flexibility of IGNN on handling heterogeneous relationships.

\begin{table*}[th]
\caption{ Node classification Micro/Macro-$F_{1}$ (\%) performance on heterogeneous network data sets.  \label{tab:Heterogeneous Networks}}
\vspace{-1mm}
\begin{center}
\begin{small}
\resizebox{0.9\textwidth}{!}{%
\begin{tabular}{@{}cCCCCCC@{}}
\toprule
\multicolumn{1}{c}{Data sets} &  \multicolumn{2}{c}{ACM} & \multicolumn{2}{c}{IMDB} & \multicolumn{2}{c}{DBLP} \\ 
\cmidrule(lr){2-3} \cmidrule(lr){4-5} \cmidrule(lr){6-7}
Metric & \text{Micro-$F_{1}$} & \text{Macro-$F_{1}$} & \text{Micro-$F_{1}$} & \text{Macro-$F_{1}$} & \text{Micro-$F_{1}$} & \text{Macro-$F_{1}$} \\
\midrule
DGI~\citep{velivckovic2018deep} & 88.1 & 88.1 & 60.6 & 59.8 & 72.0 & 72.3\\
GCN/GAT & 87.0 & 86.9 & 61.1 & 60.3 & 71.7 & 73.4  \\
DeepWalk~\citep{perozzi2014deepwalk} & 74.8 & 73.9 & 55.0 & 53.2 & 53.7 & 53.3 \\
mGCN~\citep{ma2019multi} & 86.0 & 85.8 & 63.0 & 62.3 & 71.3 & 72.5\\
HAN~\citep{wang2019heterogeneous} & 87.9 & 87.8 & 60.7 & 59.9 & 70.8 & 71.6\\
DMGI~\citep{park2019unsupervised} & 89.8 & 89.8 & 64.8 & 64.8 & \textbf{76.6} & \textbf{77.1} \\
\midrule
IGNN & \textbf{90.5} & \textbf{90.6} & \textbf{65.5} & \textbf{65.5} & 73.8 & 75.1\\
\bottomrule
\end{tabular}
}
\end{small}
\end{center}
\vspace{-6mm}
\end{table*}

\section{Conclusion}
In this paper, we present the implicit graph neural network model, a framework of recurrent graph neural networks. We describe  a sufficient condition for well-posedness based on the Perron-Frobenius theory and a projected gradient decent method for training. Similar to some other recurrent graph neural network models, implicit graph neural network captures the long-range dependency, but it carries the advantage further with a superior performance in a variety of tasks, through rigorous conditions for convergence and exact efficient gradient steps. More notably, the flexible framework extends to heterogeneous networks where it maintains its competitive performance.

\section*{Broader Impact}
GNN models are widely used on applications involving graph-structured data, including computer vision, recommender systems, and biochemical strucature discovery. Backed by more rigorous mathematical arguments, our research improves the capability GNNs of capturing the long-range dependency and therefore boosts the performance on these applications.

The improvements of performance in the applications will give rise to a better user experience of products and new discoveries in other research fields. But like any other deep learning models, GNNs runs into the problem of interpretability. The trade-off between performance and interpretability has been a topic of discussion. On one hand, the performance from GNNs benefits the tasks. On the other hand, the lack of interpretability might make it hard to recognize underlying bias when applying such algorithm to a new data set. Recent works \citep{hardt2016equality} propose to address the fairness issue by enforcing the fairness constraints.

While our research focuses on performance by capturing the long-range dependency, like many other GNNs, it does not directly tackle the fairness and interpretability aspect. We would encourage further work on fairness and interpretability on GNNs.
Another contribution of our research is on the analysis of heterogeneous networks, where the fairness on treatment of different relationships remains unexplored. The risk of discrimination in particular real-world context might require cautious handling when researchers develop models.



\begin{ack}

Funding in direct support of this work: National Key Research and Development Program of China (No. 2020AAA0107800, 2018AAA0102000), ONR Award N00014-18-1-2526, and other funding from Berkeley Artificial Intelligence Lab, Pacific Extreme Event Research Center, Genentech, Tsinghua-Berkeley Shenzhen Institute and National Natural Science Foundation of China Major Project (No. U1611461).
\end{ack}


\bibliographystyle{apalike}
\bibliography{ref.bib}

%
%
%
%
%
\clearpage
\appendix
{\centering{{\LARGE\bfseries Supplementary material}}}

\section{Kronecker Product} \label{app:kron}
For two matrices $A$ and $B$, the Kronecker product of $A\in\reals^{m\times n}$ and $B\in\reals^{p\times q}$ is denoted as $A\otimes B \in \reals^{pm\times qn}$:
$$
A\otimes B = \begin{pmatrix}
A_{11} B & \cdots & A_{1n}B \\
\vdots & \ddots & \vdots \\
A_{m1} B & \cdots & A_{mn}B \\
\end{pmatrix}.
$$
By definition of the Kronecker product, $(A\otimes B)^\top = A^\top \otimes B^\top$. 
Additionally, the following equality holds assuming compatible shapes, $(A^\top \otimes W) \vecc(X) = \vecc(WXA)$ \citep{schacke2004kronecker}, where $\vecc(X)\in \reals^{mn}$ denotes the vectorization of matrix $X\in\reals^{m\times n}$ by stacking the columns of $X$ into a single column vector of dimension $mn$. Suppose $x_i \in\reals^m$ is the $i$-th column of $X$, $\vecc(X) = [x_1^\top, \dots, x_n^\top]^\top$.

Leveraging the definition of Kronecker product and vectorization, the following equality holds, $(A^\top \otimes W) \vecc(X) = \vecc(WXA)$ \citep{schacke2004kronecker}. Intuitively, this equality reshapes $WXA$ which is linear in $X$ into a more explicit form $(A^\top \otimes W) \vecc(X)$ which is linear in $\vecc(X)$, a flattened form of $X$. Through the transformation, we place $WXA$ into the form of $Mx$. Thus, we can employ Lemma \ref{lem:wp_orig} to obtain the well-posedness conditions.

\section{Well-posedness of IGNN: Illustration, Remarks, and Proof} \label{app:wp}
\subsection{A Scalar Example} \label{app:wp_illus}
Consider the following scalar equilibrium equation (\ref{eq:scalar_example}),
\begin{equation} \label{eq:scalar_example}
    x = \text{ReLU}(wxa+u),
\end{equation}
where $x,w,a,u\in \reals$ and $\text{ReLU}(\cdot) = \max (\cdot, 0)$ is the rectified linear unit. If we set $w = a = 1$, the equation (\ref{eq:scalar_example}) will have no solutions for any $u > 0$. See Figure \ref{fig:scalar_example} for the example with $u=1$.
\begin{figure}[h]
    \centering
    \includegraphics[width=0.35\columnwidth]{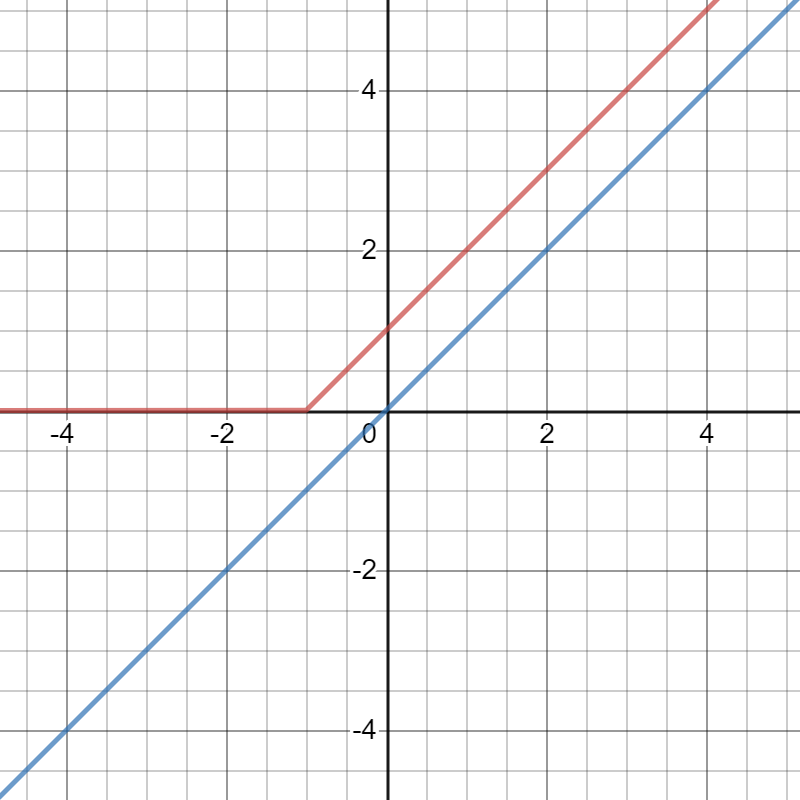}
    \caption{Plots of $x$ (red plot) and $\text{ReLU}(wxa+u) = \text{ReLU}(x+1)$ with $w=a=u=1$ (blue plot). The two plots will intersect at some point whenever a solution exists. However in this case the two plots have no intersections, meaning that there is no solution to equation (\ref{eq:scalar_example}).}
    \label{fig:scalar_example}
\end{figure}

\subsection{Detailed Explanation for Remarks} \label{app:rem}
\begin{remark}
For some non-negative adjacency matrix $A$, and arbitrary real parameter matrix $W$, $\lpf(|A^\top \otimes W|) = \lpf(A^\top \otimes |W|) = \lpf(A)\lpf(|W|)$. 
\end{remark}
The final equality of the above remark follows from the fact that, the spectrum of the Kronecker product of matrix $A$ and $B$ satisfies that $\Delta(A\otimes B) = \{\mu\lambda: \mu\in\Delta(A),\; \lambda\in\Delta(B)\}$, where $\Delta(A)$ represents the spectrum of matrix $A$. And that, the left and right eigenvalues of a matrix are the same.


We find Theorem \ref{thm:wp} to be quite general. We show that many familiar and interesting results following from it.

\begin{remark}[\ref{rem:4.2}, Contraction sufficient condition for well-posedness \citep{gori2005new}]
For any component-wise non-expansive (CONE) $\phi$, if $\mathcal{A}(X) = \phi(WXA+B)$ is a contraction of $X$ (w.r.t. vectorized norms), then $(W,A)$ is well-posed for $\phi$.
\end{remark}
The above remark follows from the fact that the contraction condition for any CONE activation map is equivalent to $\|A^\top \otimes W\| < 1$, which implies $\lpf(|A^\top \otimes W|)<1$.

\begin{remark}[\ref{rem:4.3}, Well-posedness for directed acyclic graph]
For a directed acyclic graph (DAG), let $A$ be its adjacency matrix. For any real squared $W$, we always have that $(W,A)$ is well-posed for any CONE activation map. Note that in this case $\mathcal{A}(X) = \phi(WXA+B)$ needs not be a contraction of $X$.
\end{remark}
Note that for DAG, $A$ is nilpotent ($ \lpf(A)=0 $) and thus $\lpf(|A^\top \otimes W|) = \lpf(A)\lpf(|W|) = 0$.

\begin{remark}[\ref{rem:4.4}, Sufficient well-posedness condition for k-regular graph \citep{gallicchio2019fast}]
For a k-regular graph, let $A$ be its adjacency matrix. $(W,A)$ is well-posed for any CONE activation map if $k\|W\|_2<1$.
\end{remark}
It follows from that for a k-regular graph, the PF eigenvalue of the adjacency matrix $\lpf(A)=k$. And $ \lpf(A)\lpf(|W|) \le k\|W\|_2 < 1$ guarantees well-posedness.

\begin{remark}[\ref{rem:4.5}]
For some non-negative adjacency matrix $A$, and arbitrary real parameter matrix $W$, $\|A^\top \otimes W\|_\infty = \|A^\top\|_\infty \|W\|_\infty = \|A\|_1 \|W\|_\infty$. 
\end{remark}

The above remark follows from the facts that, $\|\cdot\|_\infty$ (resp. $\|\cdot\|_1$) gives maximum row (resp. column) sum of the absolute values of a given matrix. And that, for some real matrices $A$ and $B$, $\|A\otimes B\|_\infty = \max_{i, j}\left(\sum_{k, l} |A_{ik} B_{jl}|\right) = \max_{i, j}\left(\sum_k |A_{ik}|\;  \sum_l |B_{jl}|\right) = \max_i\left(\sum_k |A_{ik}|\right)\; \max_j\left(\sum_l |B_{jl}|\right) = \|A\|_\infty \|B\|_\infty$.

\subsection{An Important Lemma for Well-posedness} \label{app:proof_lem_orig}
\begin{lemma} \label{lem:wp_orig}
If $\phi$ is component-wise non-negative (CONE), $M$ is some squared matrix and $v$ is any real vector of compatible shape, the equation $x = \phi(Mx+v)$ has a unique solution if $\lpf(|M|)<1$. And the solution can be obtained by iterating the equation. Hence, $x = \lim_{t\rightarrow \infty} x_t$.
\begin{equation} \label{eq:lemma_pf}
    x_{t+1} = \phi(Mx_t+v),\; x_0 = 0,\; t = 0, 1, \dots
\end{equation}
\end{lemma}
\begin{proof}
For existence, since $\phi$ is component-wise and non-expansive, we have that for $t \ge 1$ and the sequence $x_0, x_1, x_2, \dots$ generated from iteration (\ref{eq:lemma_pf}),
$$
|x_{t+1} - x_t| = |\phi(Mx_{t}+v) - \phi(Mx_{t-1}+v)| \le |M (x_{t} - x_{t-1})| \le |M| |x_{t} - x_{t-1}|.
$$
For $n > m \ge 1$, the following inequality follows,
\begin{equation} \label{eq:cauchy_x}
    |x_n - x_m| \le |M|^m \sum^{n-m-1}_{i = 0} |M|^i |x_1 - x_0| \le |M|^m \sum^{\infty}_{i = 0} |M|^i |x_1 - x_0| \le |M|^m w,
\end{equation}
where
$$
w:= \sum^{\infty}_{i = 0} |M|^i |x_1 - x_0| = (I-|M|)^{-1} |x_1 - x_0|.
$$
Because $\lpf(|M|)< 1$, the inverse of $I-|M|$ exists. It also follows that $\lim_{t\rightarrow \infty} |M|^t = 0$. From inequality (\ref{eq:cauchy_x}), we show that the sequence $x_0, x_1, x_2, \dots$ is a Cauchy sequence because $0 \le \lim_{m\rightarrow \infty} |x_n - x_m| \le \lim_{m\rightarrow \infty} |M|^m w = 0$. And thus the sequence converges to some solution of $x = \phi(Mx+v)$. 

For uniqueness, suppose both $x_a$ and $x_b$ satisfy $x = \phi(Mx+v)$, then the following inequality holds,
$$
0 \le |x_a - x_b| \le |M||x_a - x_b| \le \lim_{t\rightarrow \infty} |M|^t |x_a - x_b| = 0.
$$
It follows that $x_a=x_b$ and there exists unique solution to $x = \phi(Mx+v)$.
\end{proof}

\subsection{Proof of Theorem \ref{thm:wp_h}}
\begin{proof}
Similarly, we can rewrite equation (\ref{eq:dummy_equ_h}) into the following ``vectorized'' form.
$$
\vecc(X) = \phi \left( \sum_{i=1}^N (A_i^\top \otimes W_i) \vecc(X) + \sum_{i=1}^N \vecc(B_i) \right)
$$
It follows from a similar scheme as the proof of Lemma \ref{lem:wp_orig} that if $\lpf\left(\sum_{i=1}^N |A_i^\top \otimes W_i|\right) < 1$, the above equation has unique solution which can be obtained by iterating the equation. 
\end{proof}

\subsection{Proof of Theorem \ref{thm:rescale}}
\begin{proof}
The proof is based on the following formula for PF eigenvalue \citep{berman1994nonnegative}.
\begin{equation} \label{eq:rescale}
    \lpf(|W|) = \inf_{S} \: \|SWS^{-1}\|_\infty ~:~ S = \diag(s), \;\; s>0
\end{equation}

In the case where $|W|$ has simple PF eigenvalue, problem (\ref{eq:rescale}) admits positive optimal scaling factor $s>0$, a PF eigenvector of $|W|$. And we can design the equivalent IGNN $(\tilde{f_\Theta}, W', A, \tilde{b_\Omega}, \phi)$ with $\|W'\|_\infty < \lpf(A)^{-1}$ by rescaling:
$$
\Tilde{f_\Theta}(\cdot) = f_\Theta(S^{-1}\;\cdot), \;\;\;\; W' = S W S^{-1}, \;\;\;\; \Tilde{b_\Omega}(\cdot) = S b_\Omega(\cdot),
$$
where $S = \diag(s)$.
\end{proof}

\section{Examples of IGNN} \label{app:example_ignn}
In this section we introduce some examples of the variation of IGNN. 
\paragraph{Multi-layer Setup.}
It is straight forward to extend IGNN to a multi-layer setup with several sets of $W$ and $\Omega$ parameters for each layer. For conciseness, we use the ordinary graph setting. By treating the fixed-point solution $X_{l-1}$ of the $(l-1)$-th layer as the input $U_{l}$ to the $l$-th layer of equilibrium equation, a multi-layer formulation of IGNN with a total of $L$ layers is created.
\begin{equation}\label{eq:ignn_ml}
\begin{aligned}
    \hat{Y} &= f_\Theta (X_L), \\
    X_L &= \phi_L(W_L X_L A + b_{\Omega_L}(X_{L-1})), \\
    &\vdots\\
    X_l &= \phi_l(W_l X_l A + b_{\Omega_l}(X_{l-1})), \\
    &\vdots\\
    X_1 &= \phi_1(W_1 X_1 A + b_{\Omega_1}(U)),
\end{aligned}
\end{equation}
where $\phi_1, \dots, \phi_L$ are activation functions. We usually assume that CONE property holds on them. And $(W_l, \Omega_l)$ is the set of weights for the $l$-th layer. Thus the multi-layer formulation (\ref{eq:ignn_ml}) with parameters $(W_l, l = 1, \dots, L,\; A)$ is well-posed (\ie \; gives unique prediction $\hat{Y}$ for any input $U$) when $(W_l, A)$ is well-posed for $\phi_l$ for any layer $l$. This is true since the well-posedness for a layer guarantees valid input for the next layer. Since all layers are well-posed, the formulation will give unique final output for any input of compatible shape. FDGNN \citep{gallicchio2019fast} uses a similar multi-layer formulation for graph classification but is only partially trained in practive.

In terms of the affine input function, $b_{\Omega}(U) = \Omega U A$ is a good choice. 
We show that the multi-layer IGNN with such $b_\Omega$ is equivalent to a single layer IGNN (\ref{eq:ignn}) with higher dimensions, the same $A$ matrix and $f_\Theta$ function. The new activation map is given by $\phi = (\phi_L, \dots, \phi_l, \dots, \phi_1)$. Although $\phi$ is written in a block-wise form, they still operate on entry level and remain non-expansive. Thus the well-posedness results still hold. The new $\tilde{W}$ and $\tilde{b_\Omega}$ write,
\begin{equation}
    \tilde{W} = \begin{pmatrix}
W_L & \Omega_L &  &  \\
  & \ddots & \ddots & \\
  &   & \ddots & \Omega_2 \\
 &   &  &  W_1
\end{pmatrix}, \;\;\;\;
\tilde{b_\Omega}(U) = \begin{pmatrix}
  0\\
  \vdots \\
  0 \\
 \Omega_1
\end{pmatrix} U A.
\end{equation}

\paragraph{Special Cases.}
Many existing GNN formulations including convolutional and recurrent GNNs can be treated as special cases of IGNN. We start by showing that GCN \citep{kipf2016semi}, a typical example of convolutional GNNs, is indeed an IGNN. We give the matrix representation of a 2-layer GCN as follows,
\begin{equation} \label{eq:gcn}
    \begin{aligned}
    \hat{Y} &= W_2 X_1 A, \\
    X_1 &= \phi_1 (W_1 U A),
    \end{aligned}
\end{equation}
where $A$ is the renormalized adjacency matrix; $W_1$ and $W_2$ are weight parameters; $\phi_1$ is a CONE activation map for the first layer; and $X_1$ is the hidden representation of first layer. We show that GCN (\ref{eq:gcn}) is in fact a special case of IGNN by constructing an equivalent single layer IGNN (\ref{eq:ignn}) with the same $A$ matrix. 
\begin{subequations}
\begin{align}
  \hat{Y} &= \tilde{f_\Theta}(\tilde{X}) ,\\
  \tilde{X} &= \phi(\tilde{W}\tilde{X}A+\tilde{b_\Omega}(U)).
\end{align}
\end{subequations}
The new state $\tilde{X} = (X_2, X_1)$. The new activation map is given by $\phi = (\phi_1,\mathbb{I})$, where $\mathbb{I}$ represents an identity map. And the new $\tilde{W}$, $\tilde{b_\Omega}$, and $\tilde{f_\Theta}(\tilde{X})$ are,
\begin{equation}
    \tilde{W} = \begin{pmatrix}
0 & W_2\\
0 & 0
\end{pmatrix}, \;\;\;\;
\tilde{b_\Omega}(U) = \begin{pmatrix}
  0\\
 W_1
\end{pmatrix} U A, \;\;\;\;
\tilde{f_\Theta}(\tilde{X}) = \begin{pmatrix}
  I\\
 0
\end{pmatrix} \tilde{X}.
\end{equation}
This reformulation of single layer IGNN also extends to multi-layer GCNs with more than 2 layers as well as other convolutional GNNs. Note that the new $\tilde{W}$ for the equivalent single layer IGNN is always strictly upper triangular. Thus $|\tilde{W}|$ has only $0$ eigenvalue. As a result, $\lpf(|A^\top \otimes W|) = \lpf(A)\lpf(|W|) = 0$ and the sufficient condition for well-posedness is always satisfied.

Another interesting special case is SSE~\citep{dai2018learning}, an example of recurrent GNN, that is given by
\begin{equation} \label{eq:sse}
\begin{aligned}
    \hat{Y} &= W_2 X, \\
    X &= \phi (W_{1r} W_2 X A + W_{1u} U A + W_{1u}' U),
\end{aligned}
\end{equation}
which can be easily converted into a single layer IGNN with the same $A$ matrix and CONE activation $\phi$. The new $\tilde{W}$, $\tilde{b_\Omega}$, and $\tilde{f_\Theta}(X)$ are,
\begin{equation}
    \tilde{W} = W_{1r} W_2, \;\;\;\; \tilde{b_\Omega}(U) = W_{1u} U A + W_{1u}' U, \;\;\;\; \tilde{f_\Theta}(X) = W_2 X.
\end{equation}


\section{Implicit differentiation for IGNN} \label{app:gradient}
To compute gradient of $\mathcal{L}$ from the training problem (\ref{eq:training}) w.r.t. a scalar $q \in W \cup \Omega$, we can use chain rule. It follows that,

\begin{equation} \label{eq:nabla_q}
    \nabla_q \mathcal{L} = \left\langle \frac{\partial X}{\partial q}, \nabla_X \mathcal{L}\right \rangle,
\end{equation}
where $\nabla_X \mathcal{L}$ can be easily calculated through modern autograd frameworks. But $\frac{\partial X}{\partial q}$ is non-trivial to obtain because $X$ is only implicitly defined. Fortunately, we can still leverage chain rule in this case by carefully taking the ``implicitness'' into account.

To avoid taking derivatives of matrices by matrices, we again introduce the vectorized representation $\vecc(\cdot)$ of matrices. The vectorization of a matrix $X\in\reals^{m\times n}$, denoted $\vecc(X)$, is obtained by stacking the columns of $X$ into one single column vector of dimension $mn$. For simplicity, we use $\vec{X} := \vecc(X)$ and $\nabla_{\vec{X}}\mathcal{L} = \vecc(\nabla_X \mathcal{L})$ as a short hand notation of vectorization. 

\begin{equation} \label{eq:x_q_chain}
    \frac{\partial \vec{X}}{\partial q} = \frac{\partial \vec{X}}{\partial \vec{Z}} \cdot \frac{\partial \vec{Z}}{\partial q},
\end{equation}

where $Z = WXA + b_\Omega(U)$ ($\vec{Z} = (A^\top \otimes W) \vec{X} + \overrightarrow{b_\Omega(U)}$) \emph{assuming fixed $X$}. Unlike $X$ in equation (\ref{eq:ignn_equ}), $Z$ is not implicitly defined and should only be considered as a closed evaluation of $Z = WXA + b_\Omega(U)$ assuming $X$ doesn't change depending on $Z$. In some sense, the $Z$ in equation (\ref{eq:x_q_chain}) doesn't equal to $WXA + b_\Omega(U)$. However, the closeness property will greatly simplify the evaluation of $\frac{\partial \vec{Z}}{\partial q}$. It turns out that we can still employ chain rule in this case to calculate $\frac{\partial \vec{X}}{\partial \vec{Z}}$ for such $Z$ by taking the change of $X$ before hand into account as follows,
\begin{equation} \label{eq:x_z_chain}
    \frac{\partial \vec{X}}{\partial \vec{Z}} = \frac{\partial \phi (\vec{Z})}{\partial \vec{Z}} + \frac{\partial \phi \left((A^\top\otimes W)\vec{X} + \overrightarrow{b_\Omega(U)}\right)}{\partial \vec{X}} \cdot \frac{\partial \vec{X}}{\partial \vec{Z}},
\end{equation}
where the second term accounts for the change in $X$ that was ignored in $Z$. Another way to view this calculation is to right multiply $\frac{\partial \vec{Z}}{\partial q}$ on both sides of equation (\ref{eq:x_z_chain}), which gives the chain rule evaluation of $\frac{\partial \vec{X}}{\partial q}$ that takes the gradient flowing back to $X$ into account:
$$
\frac{\partial \vec{X}}{\partial q} = \frac{\partial \phi \left((A^\top\otimes W)\vec{X} + \overrightarrow{b_\Omega(U)}\right)}{\partial q} + \frac{\partial \phi \left((A^\top\otimes W)\vec{X} + \overrightarrow{b_\Omega(U)}\right)}{\partial \vec{X}} \cdot \frac{\partial \vec{X}}{\partial q}.
$$

The equation (\ref{eq:x_z_chain}) can be simplified as follows,
\begin{align} \label{x_z_simple}
    \frac{\partial \vec{X}}{\partial \vec{Z}} &= (I-J)^{-1} \Tilde{D}, \\
    J &= \frac{\partial \phi\left( (A^\top \otimes W) \vec{X} + \overrightarrow{b_{\Omega} (U)}\right)}{\partial \vec{X}} = \Tilde{D}(A^\top\otimes W),  \notag
\end{align}
where $\Tilde{D} = \frac{\partial \phi(\vec{Z})}{\partial \vec{Z}} = \diag \left(\phi'\left((A^\top \otimes W) \vec{X} + \overrightarrow{b_\Omega (U)}\right)\right).$ Now we can rewrite equation (\ref{eq:nabla_q}) as

\begin{align}
    \nabla_q \mathcal{L} &= \left\langle \frac{\partial \vec{Z}}{\partial q}, \nabla_{\vec{Z}} \mathcal{L} \right\rangle,  \label{eq:nabla_qq} \\
    \nabla_{\vec{Z}} \mathcal{L} &= \left(\frac{\partial \vec{X}}{\partial \vec{Z}}\right)^\top \nabla_{\vec{X}} \mathcal{L},  \label{eq:nabla_z}
\end{align}
which is equivalent to equation (\ref{eq:q_grad}). $\nabla_{\vec{Z}}\mathcal{L}$ should be interpreted as the direction of steepest change of $\mathcal{L}$ for $Z = WXA + b_\Omega(U)$ \emph{assuming fixed $X$}. Plugging equation (\ref{eq:x_z_chain}) to (\ref{eq:nabla_z}), we arrive at the following equilibrium equation (equivalent to equation (\ref{eq:gradient_z}))
\begin{align} \label{eq:gradient_z_app}
    \nabla_{\vec{Z}} \mathcal{L} &= \Tilde{D} (A\otimes W^\top) \nabla_{\vec{Z}} \mathcal{L} + \Tilde{D}\; \nabla_{\vec{X}} \mathcal{L}, \notag \\
    \nabla_Z \mathcal{L} &= D \odot \left(W^\top \nabla_Z\mathcal{L}\; A^\top + \nabla_X \mathcal{L}\right),
\end{align}
where $D = \phi'(WXA+ b_\Omega(U))$. Interestingly, $\nabla_Z \mathcal{L}$ turns out to be given as a solution of an equilibrium equation particularly similar to equation (\ref{eq:ignn_equ}) in the IGNN ``forward'' pass. In fact, we can see element-wise multiplication with $D$ as a CONE ``activation map'' $\Tilde{\phi}(\cdot) = D\odot(\cdot)$. And it follows from Section \ref{sub:wellposedness} that if $\lpf(W)\lpf(A)<1$, then $\lpf(W^\top)\lpf(A^\top)<1$ and $\nabla_Z \mathcal{L}$ can be uniquely determined by iterating the above equation (\ref{eq:gradient_z_app}). Although the proof will be more involved, if $(W,A)$ is well-posed for any CONE activation map, we can conclude that equilibrium equation (\ref{eq:gradient_z_app}) is also well-posed for $\Tilde{\phi}$ where $\phi$ can be any CONE activation map.

Finally, by plugging the evaluated $\nabla_Z\mathcal{L}$ into equation (\ref{eq:nabla_qq}), we get the desired gradients. Note that it is also possible to obtain gradient $\nabla_U \mathcal{L}$ by setting the $q$ in the above calculation to be $q\in U$. This is valid because we have no restrictions on selection of $q$ other than that it is not $X$, which is assumed fixed. Following the chain rule, we can give the closed form formula for $\nabla_W \mathcal{L}$, $\nabla_{\omega} \mathcal{L}, \omega\in\Omega$, and $\nabla_{u} \mathcal{L}, u\in U$.

$$
\nabla_W \mathcal{L} = \nabla_Z \mathcal{L}\;A^\top X^\top, \;\;\;\; \nabla_{\omega} \mathcal{L} = \left\langle \frac{\partial b_\Omega (U)}{\partial \omega}, \nabla_Z\mathcal{L} \right\rangle ,\;\;\;\; \nabla_{u} \mathcal{L} = \left\langle \frac{\partial b_\Omega (U)}{\partial u}, \nabla_Z\mathcal{L} \right\rangle.
$$
\paragraph{Heterogeneous Network Setting}
We start by giving the training problem for heterogeneous networks similar to training problem (\ref{eq:training}) for ordinary graphs,
\begin{align} \label{eq:training_h}
    \min_{\Theta, W,\Omega} \;\; &\mathcal{L}(Y, f_\Theta(X)) \notag \\
    \text{s.t.} \;\; & X = \phi \left( \sum_{i=1}^N (W_iXA_i+b_{\Omega_i}(U_i))\right), \\
    & \;\;\;\; \sum_{i=1}^N \|A_i\|_1 \|W_i\|_\infty \le \kappa \notag.
\end{align}

The training problem can be solved again using projected gradient descent method where the gradient of $W_i$ and $\Omega_i$ for $i\in R$ can be obtained with implicit differentiation. Using chain rule, we write the gradient of a scalar $q\in \bigcup_i (W_i \cup \Omega_i)$,

\begin{equation} \label{eq:q_grad_h}
    \nabla_{q} \mathcal{L} = \left\langle \frac{\partial \left(\sum_{i=1}^N (W_iXA_i+b_{\Omega_i}(U_i))\right)}{\partial q}, \nabla_Z \mathcal{L}\right\rangle,
\end{equation}

where $Z = \sum_{i=1}^N (W_iXA_i+b_{\Omega_i}(U_i))$ and $\nabla_Z \mathcal{L}$ in equation (\ref{eq:q_grad_h}) should be interpreted as ``direction of fastest change of $\mathcal{L}$ for $Z$ \emph{assuming fixed $X$}''. Similar to the derivation in ordinary graphs setting, such notion of $\nabla_Z \mathcal{L}$ enables convenient calculation of $\nabla_{q} \mathcal{L}$. And the vectorized gradient w.r.t. $Z$ can be expressed as a function of the vectorized gradient w.r.t. $X$:
\begin{align}
    \nabla_{\vec{Z}} \mathcal{L} &= \left(\frac{\partial \vec{X}}{\partial\vec{Z}}\right)^\top \nabla_{\vec{X}} \mathcal{L} \label{eq:nabla_z_h} \\
    \frac{\partial \vec{X}}{\partial\vec{Z}} &= \frac{\partial \phi (\vec{Z})}{\partial \vec{Z}} + \frac{\partial \phi \left(\sum_{i=1}^N \left( (A_i^\top\otimes W_i)\vec{X} + \overrightarrow{b_{\Omega_i}(U_i)}\right)\right)}{\partial \vec{X}} \cdot \frac{\partial \vec{X}}{\partial \vec{Z}} \notag \\
    &= (I-J)^{-1} \Tilde{D} \label{eq:x_z_chain_h} \\
    J &= \frac{\partial \phi\left(\sum_{i=1}^N \left( (A_i^\top\otimes W_i)\vec{X} + \overrightarrow{b_{\Omega_i}(U_i)}\right)\right)}{\partial \vec{X}} = \Tilde{D}\sum_{i=1}^N (A_i^\top\otimes W_i),  \notag
\end{align}
where $\Tilde{D} = \frac{\partial \phi(\vec{Z})}{\partial \vec{Z}} = \diag \left(\phi'\left(\sum_{i=1}^N \left( (A_i^\top\otimes W_i)\vec{X} + \overrightarrow{b_{\Omega_i}(U_i)}\right)\right)\right).$ Plugging the expression (\ref{eq:x_z_chain_h}) into (\ref{eq:nabla_z_h}), we arrive at the following equilibrium equation for $\nabla_{\vec{Z}}\mathcal{L}$ and $\nabla_{Z}\mathcal{L}$,
\begin{align} \label{eq:gradient_z_app_h}
    \nabla_{\vec{Z}} \mathcal{L} &= \Tilde{D} \sum_{i=1}^N (A_i\otimes W_i^\top) \nabla_{\vec{Z}} \mathcal{L} + \Tilde{D}\; \nabla_{\vec{X}} \mathcal{L} \notag \\
    \nabla_Z \mathcal{L} &= D \odot \left(\sum_{i=1}^N (W_i^\top \nabla_Z\mathcal{L} A_i^\top) + \nabla_X \mathcal{L}\right),
\end{align}
where $D = \phi'\left(\sum_{i=1}^N (W_iXA_i+ b_{\Omega_i}(U_i))\right)$. Not surprisingly, the equilibrium equation (\ref{eq:gradient_z_app_h}) again appears to be similar to the equation (\ref{eq:ignn_h_equ}) in the IGNN ``forward'' pass. We can also view element-wise multiplication with $D$ as a CONE ``activation map'' $\Tilde{\phi}(\cdot) = D\odot(\cdot)$. And it follows from Section \ref{sub:wellposedness} that if  $\lpf(|A^\top \otimes W|) < 1$, then  $\lpf(|A \otimes W^\top|) < 1$ and $\nabla_Z \mathcal{L}$ can be uniquely determined by iterating the above equation (\ref{eq:gradient_z_app}). It also holds that if $(W_i, A_i, i\in \{1, \dots, N\})$ is well-posed for any CONE activation $\phi$, then 
we can conclude that equilibrium equation (\ref{eq:gradient_z_app_h}) is also well-posed for $\Tilde{\phi}$ where $\phi$ can be any CONE activation map.

Finally, by plugging the evaluated $\nabla_Z\mathcal{L}$ into equation (\ref{eq:q_grad_h}), we get the desired gradients. It is also possible to obtain gradient $\nabla_{U_i} \mathcal{L}$ by setting the $q$ in the above calculation to be $q\in \bigcup_i U_i$. This is valid because we have no restrictions on selection of $q$ other than that it is not $X$, which is assumed fixed.

After the gradient step, the projection to the tractable condition mentioned in Section \ref{sub:tractable} can be done approximately by assigning $\kappa_i$ for each relation $i\in R$ and projecting $W_i$ onto $\mathcal{C}_i = \{\|W_i\|_\infty \le \kappa_i/\|A\|_1\}$. Ensuring $\sum_i \kappa_i = \kappa < 1$ will guarantee that the PF condition for heterogeneous network is satisfied. However, empirically, setting $\kappa_i < 1$ with $\sum_i \kappa_i > 1$ in some cases is enough for the convergence property to hold for the equilibrium equations.

\section{More on Experiments} \label{app:more experiments}

In this section, we give detailed information about the experiments we conduct. 

For preprocessing, we apply the \emph{renormalization trick} consistent with GCN~\citep{kipf2016semi} on the adjacent matrix of all data sets.

In terms of hyperparameters, unless otherwise specified, for IGNN, we use affine transformation $b_\Omega(U) = \Omega U A$; linear output function $f_\Theta(X) = \Theta X$; ReLU activation $\phi(\cdot) = \max(\cdot, 0)$; learning rate 0.01; dropout with parameter $0.5$ before the output function; and $\kappa = 0.95$. We tune layers, hidden nodes, and $\kappa$ through grid search. The hyperparameters for other baselines are consistent with that reported in their papers. Results with identical experimental settings are reused from previous works.


\subsection{Synthetic Chains Data Set}

We construct a synthetic node classification task to test the capability of models of learning to gather information from distant nodes. We consider the chains directed from one end to the other end with length $l$ (\ie \; $l+1$ nodes in the chain). For simplicity, we consider binary classification task with 2 types of chains. Information about the type is only encoded as 1/0 in first dimension of the feature (100d) on the starting end of the chain. The labels are provided as one-hot vectors (2d). In the data set 
we choose chain length $l=9$ and 20 chains for each class with a total of 400 nodes. The training set consists of 20 data points randomly picked from these nodes in the total 40 chains. Respectively, the validation set and test set have 100 and 200 nodes. 

A single-layer IGNN is implemented with 16 hidden unites and weight decay of parameter $5\times 10^{-4}$ for all chains data sets with different $l$. Four representative baselines are chosen: Stochastic Steady-state Embedding (\textbf{SSE})~\citep{dai2018learning}, Graph 
Convolutional Network (\textbf{GCN})~\citep{kipf2016semi}, Simple Graph Convolution (\textbf{SGC})~\citep{pmlrv97wu19e} and Graph Attention Network (\textbf{GAT})~\citep{velivckovic2017graph}. They all use the same hidden units and weight decay as IGNN. For (\textbf{GAT}), 8 head attention is used. For (\textbf{SSE}), we use the embedding directly as output and fix-point iteration $n_h = 8$, as suggested \citep{dai2018learning}.

As mentioned in Section \ref{sec:numerical}, convolutional GNNs with $T=2$ cannot capture the dependency with a range larger than $2$-hops. To see how convolutional GNNs capture the long-range dependency as $T$ grows, we give an illustration of Micro-$F_1$ verses $T$ for the selected baselines in Figure \ref{fig:chains_T}. From the experiment, we find that convolutional GNNs cannot capture the long-range dependency given larger $T$. This might be a result of the limited number of training nodes in this chain task. As $T$ grows, convolutional GNNs experience an explosion of number of parameters to train. Thus the training data becomes insufficient for these models as the number of parameters increases.

\begin{figure}
    \centering
    \begin{minipage}{0.45\textwidth}
		\centering
		\caption{Chains with $l=9$. Traditional methods fail even with more iterations. }
		\includegraphics[width=0.8\textwidth]{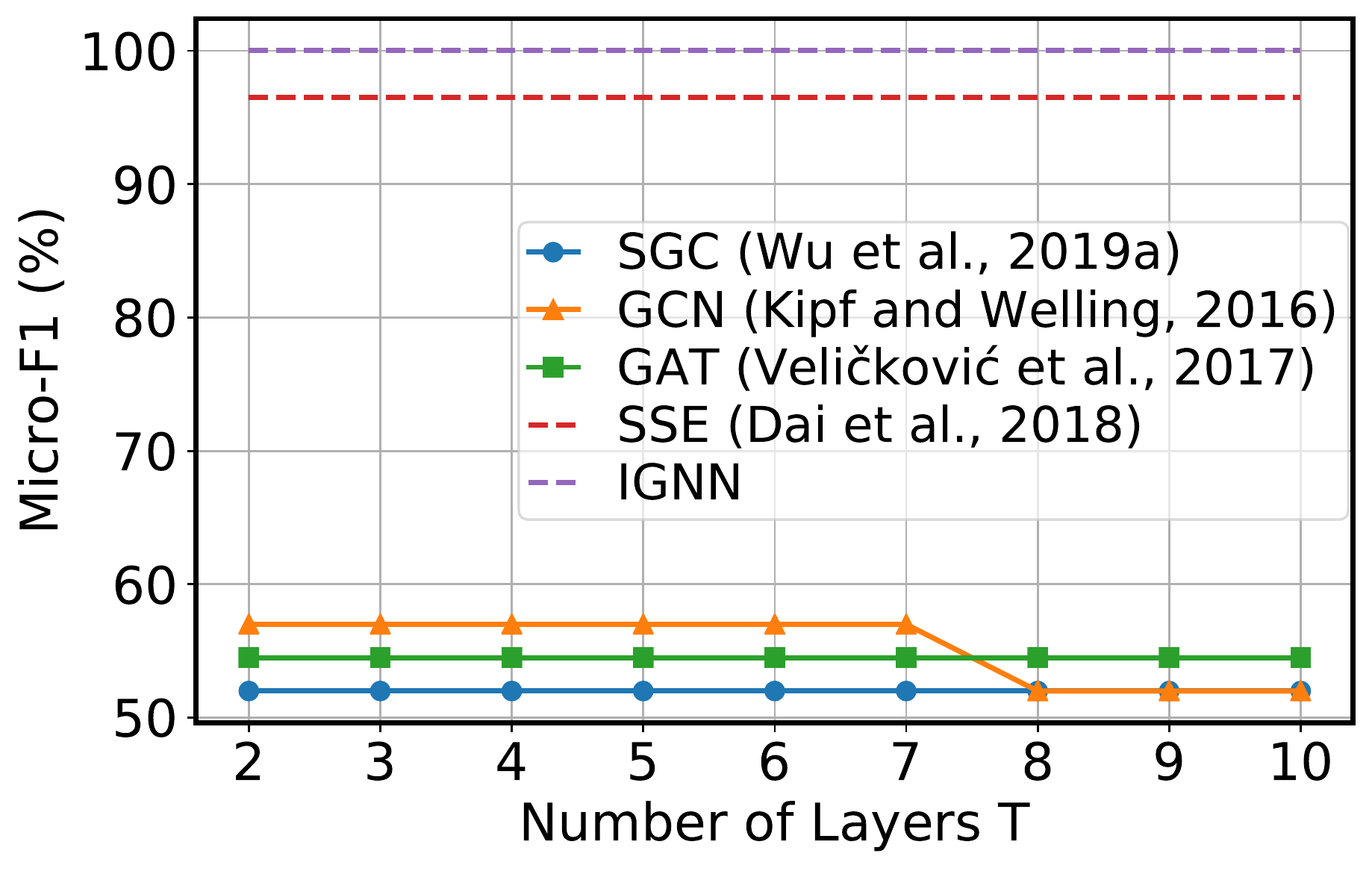}
		\label{fig:chains_T}
	\end{minipage}
\end{figure}

\subsection{Node Classification}
For node classification task, we consider the applications under both transductive (Amazon) \citep{yang2015defining} and inductive (PPI) \citep{hamilton2017inductive} settings.
Transducive setting is where the model has access to the feature vectors of all nodes during training, while inductive setting is where the graphs for testing remain completely unobserved during training.
The statistics of the data sets can be found in Table~\ref{tab:node classification datasets stats}.

\begin{table}[!t]
\centering
\caption{The overview of data set statistics in node classification tasks.}
\vspace{2mm}
\resizebox{0.98\textwidth}{!}{%
\begin{tabular}{l c c c c c}
\toprule
\textbf{Data set} & \textbf{\# Nodes} & \textbf{\# Edges} & \textbf{\# Labels} & \textbf{Label type} & \textbf{Graph type}  \\[0.05em]\hline \\[-0.8em]
\textbf{Amazon (transductive)} & 334,863 & 925,872 & $58$ & Product type & Co-purchasing \\
\textbf{PPI (inductive)} & 56,944 & 818,716 & $121$ & Bio-states & Protein \\

\bottomrule
\end{tabular}
\label{tab:node classification datasets stats}
}
\end{table}

For experiments on Amazon, we construct a one-layer IGNN with 128 hidden units. 
No weight decay is utilized. 
The hyper parameters of baselines are consistent with~\citep{yang2015defining,dai2018learning}.

For experiments on PPI,
a five-layer IGNN model is applied for this multi-label classification tasks with hidden units as [1024, 512, 512, 256, 121] and $\kappa = 0.98$ for each layer. In addition, four MLPs are applied between the first four consecutive IGNN layers. We use the identity output function. Neither weight decay nor dropout is employed. We keep the experimental settings of baselines consistent with~\citep{velivckovic2017graph,dai2018learning,kipf2016semi,hamilton2017inductive}.
\subsection{Graph Classification}
For graph classification, 5 bioinformatics data sets are employed with information given in Table~\ref{tab:graph classification}. 
We compare IGNN with a comprehensive set of baselines, including a variety of GNNs: Deep Graph Convolutional Neural Network (\textbf{DGCNN})~\citep{zhang2018end}, Diffusion-Convolutional Neural Networks (\textbf{DCNN})~\citep{atwood2016diffusion}, Fast and Deep
Graph Neural Network (\textbf{FDGNN})~\citep{gallicchio2019fast}, \textbf{GCN}~\citep{kipf2016semi} and Graph Isomorphism Network (\textbf{GIN})~\citep{xu2018powerful}, and a number of state-of-the-art graph kernels: Graphlet Kernel (\textbf{GK})~\citep{shervashidze2009efficient}, Random-walk Kernel (\textbf{RW})~\citep{gartner2003graph}, Propagation Kernel (\textbf{PK})~\citep{neumann2016propagation} and Weisfeiler-Lehman Kernel (\textbf{WL})~\citep{shervashidze2011weisfeiler}. We reuse the results from literatures~\citep{xu2018powerful,gallicchio2019fast} since the same experimental settings are maintained. 

As of IGNN, a three-layer IGNN is constructed for comparison with the hidden units of each layer as 32 and $\kappa = 0.98$ for all layers. We use an MLP as the output function. Besides, batch normalization is applied on each hidden layer. Neither weight decay nor dropout is utilized.

\subsection{Heterogeneous Networks} 

For heterogeneous networks, three data sets are chosen (ACM, IMDB, and DBLP). Consistent with previous works~\citep{park2019unsupervised}, we use the the publicly available ACM data set~\citep{wang2019heterogeneous}, preprocessed DBLP and IMDB data sets~\citep{park2019unsupervised}. For ACM and DBLP data sets, the nodes are papers and the aim is to classify the papers into three classes (Database, Wireless Communication, Data Mining), and four classes (DM, AI, CV, NLP)\footnote{\textbf{DM}: KDD,WSDM,ICDM, \textbf{AI}: ICML,AAAI,IJCAI, \textbf{CV}: CVPR, \textbf{NLP}: ACL,NAACL,EMNLP}, respectively. For IMDB data set, the nodes are movies and we aim to classify these movies into three classes (Action, Comedy, Drama). The detailed information of data sets can be referred to Table~\ref{tab:heterogeneous datasets stats}. The preprocessing procedure and splitting method on three data sets keep consistent with~\citep{park2019unsupervised}.

State-of-the-art baselines are selected for comparison with IGNN, including no-attribute network embedding: \textbf{DeepWalk}~\citep{perozzi2014deepwalk}, attributed network embedding: \textbf{GCN}, \textbf{GAT} and \textbf{DGI}~\citep{velivckovic2018deep}, and attributed multiplex network embedding: \textbf{mGCN}~\citep{ma2019multi}, \textbf{HAN}~\citep{wang2019heterogeneous} and \textbf{DMGI}~\citep{park2019unsupervised}. Given the same experimental settings, we reuse the results of baselines from~\citep{park2019unsupervised}.

A one-layer IGNN with hidden units as 64 is implemented on all data sets. Similar to \textbf{DMGI}, a weight decay of parameter $0.001$ is used. For ACM, $\kappa = (0.55, 0.55)$ is used for Paper-Author and Paper-Subject relations. For IMDB, we select $\kappa = (0.5, 0.5)$ for Movie-Actor and Movie-Director relations. For DBLP, $\kappa = (0.7, 0.4)$ is employed for Paper-Author and Paper-Paper relations. As mentioned in Appendix \ref{app:gradient}, in practice, the convergence property can still hold when $\sum_i \kappa_i >1$.


\begin{table*}[t]
	\centering
	\small
	
	\caption{Statistics of the data sets for heterogeneous graphs \citep{park2019unsupervised}. The node attributes are bag-of-words of text. Num. labeled data denotes the number of nodes involved during training.}
	\setlength\tabcolsep{1pt}
	\resizebox{\textwidth}{!}{%
	\begin{tabular}{c|c|c|c|c|c|c|c|c|c}
		& \begin{tabular}[x]{@{}c@{}}Relations \\(A-B)\end{tabular}   & Num. A   & Num. B & Num. A-B & Relation type & \begin{tabular}[x]{@{}c@{}}Num. \\relations\end{tabular} & \begin{tabular}[x]{@{}c@{}}Num. \\node attributes\end{tabular}  &  \begin{tabular}[x]{@{}c@{}}Num. \\labeled data\end{tabular}       & \begin{tabular}[x]{@{}c@{}}Num.  \\classes\end{tabular}   \\
		\hline
		\multirow{2}{*}{ACM}  & \underline{P}aper-\underline{A}uthor   & 3,025  & 5,835    & 9,744  & \underline{P}-\underline{A}-\underline{P} & 29,281  & \multirow{2}{*}{\begin{tabular}[x]{@{}c@{}}1,830\vspace{-0.5ex} \\(Paper abstract)\end{tabular}}  & \multirow{2}{*}{600}                      & \multirow{2}{*}3    \\
		& \underline{P}aper-\underline{S}ubject  & 3,025  & 56       & 3,025  & \underline{P}-\underline{S}-\underline{P}  & 2,210,761 &                                   &                                 &  \\
		\hline
		\multirow{2}{*}{IMDB} & \underline{M}ovie-\underline{A}ctor    & 3,550  & 4,441    & 10,650  & \underline{M}-\underline{A}-\underline{M}  & 66,428  & \multirow{2}{*}{\begin{tabular}[x]{@{}c@{}}1,007\vspace{-0.5ex} \\(Movie plot)\end{tabular}}           & \multirow{2}{*}{300}                & \multirow{2}{*}3 \\
		& \underline{M}ovie-\underline{D}irector & 3,550  & 1,726    & 3,550  & \underline{M}-\underline{D}-\underline{M}   & 13,788  &                                   &                                   & \\
		\hline
		\multirow{3}{*}{DBLP} & \underline{P}aper-\underline{A}uthor   & 7,907  & 1,960    & 14,238 & \underline{P}-\underline{A}-\underline{P}  & 144,783 & \multirow{3}{*}{\begin{tabular}[x]{@{}c@{}}2,000 \\(Paper abstract)\end{tabular}}            & \multirow{3}{*}{80}              & \multirow{3}{*}4   \\
		& \underline{P}aper-\underline{P}aper    & 7,907  & 7,907    & 10,522   & \underline{P}-\underline{P}-\underline{P}   & 90,145  &                                   &                               &  \\
		& \underline{A}uthor-\underline{T}erm    & 1,960  & 1,975    & 57,269  & \underline{P}-\underline{A}-\underline{T}-\underline{A}-\underline{P}   & 57,137,515    &                                &                             & \\
	\end{tabular}
	}
	\label{tab:heterogeneous datasets stats}
\end{table*}

\subsection{Over-smoothness} \label{app:oversmooth}
Convolutional GNNs has suffered from over-smoothness when the model gets deep. An interesting question to ask is whether IGNN suffers from the same issue and experience performance degradation in capturing long-range dependency with its "infinitively deep" GNN design.

In an effort to answer this question, we compared IGNN against two latest convolutional GNN models that solve the over-smoothness issue, GCNII \cite{chen2020simple} and DropEdge \cite{rong2020dropedge}. We use the same experimental setting as the Chains experiment in section \ref{sec:numerical}. Both GCNII and DropEdge are implemented with 10-layer and is compared with IGNN in capturing long-range dependency. The result is reported in Figure \ref{fig:chains_oversmooth}. We observe that IGNN consistently outperforms both GCNII and DropEdge as the chains gets longer. The empirical result suggest little suffering from over-smoothness for recurrent GNNs.

\begin{figure}[htbp]
	\centering
	\begin{minipage}{0.45\textwidth}
		\centering
		\caption{Micro-$F_{1}$ (\%) performance with respect to the length of the chains.}
		\label{fig:chains_oversmooth}
		\includegraphics[width=0.95\textwidth]{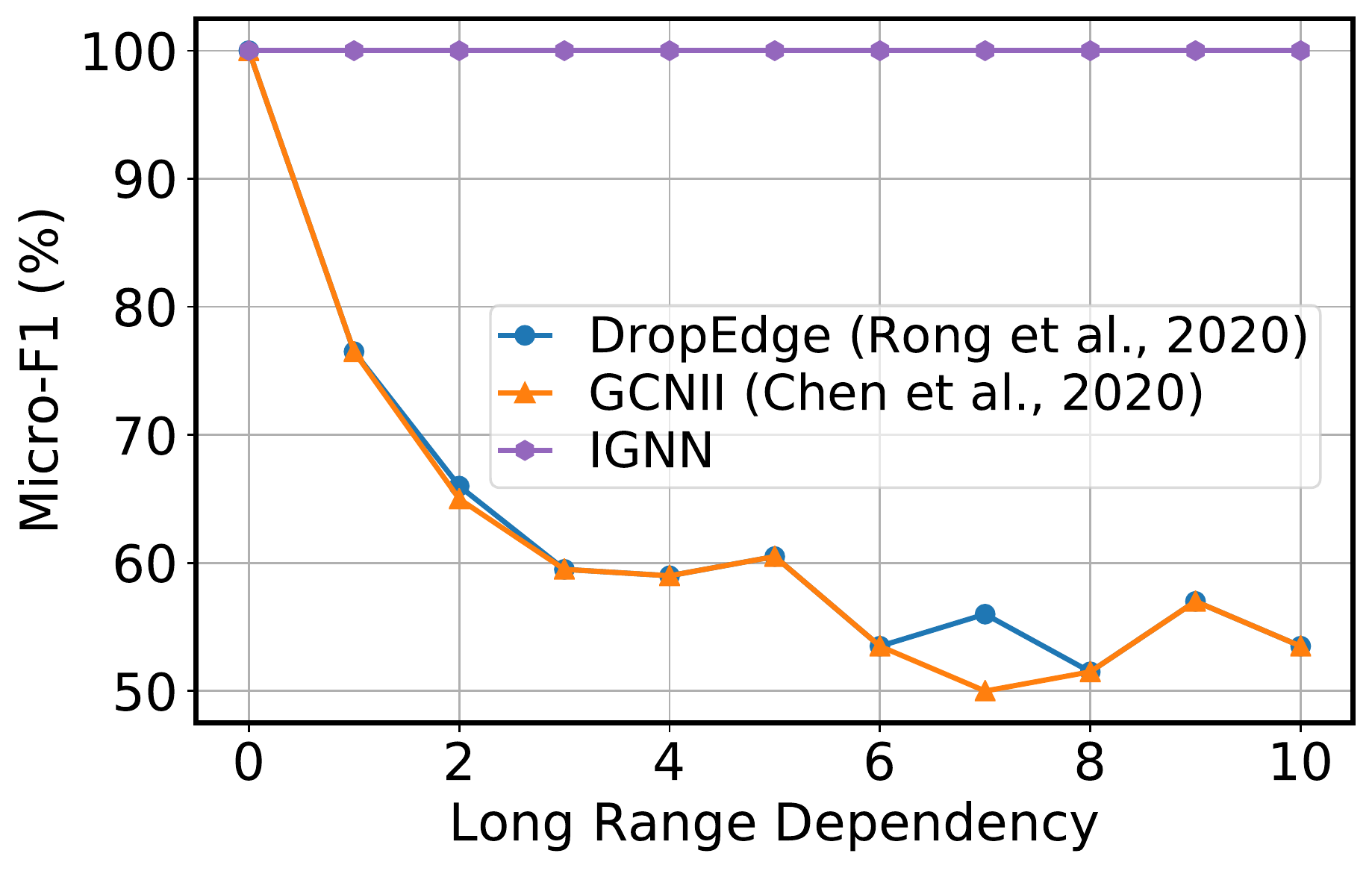}
	\end{minipage}
\end{figure}

\end{document}